\documentclass{article}

\usepackage{microtype}
\usepackage{graphicx}
\usepackage{booktabs} % for professional tables

\usepackage{hyperref}

\usepackage[accepted]{util/icml2024}

\usepackage{amsmath}
\usepackage{amssymb}
\usepackage{mathtools}
\usepackage{amsthm}
\usepackage{mathabx}
\usepackage{enumitem}

\usepackage[capitalize,noabbrev]{cleveref}

\usepackage[caption=false]{subfig}
\usepackage{wrapfig}

\theoremstyle{plain}
\newtheorem{theorem}{Theorem}[section]
\newtheorem{proposition}[theorem]{Proposition}
\newtheorem{lemma}[theorem]{Lemma}
\newtheorem{corollary}[theorem]{Corollary}
\theoremstyle{definition}

\theoremstyle{remark}
\newtheorem{remark}[theorem]{Remark}
\usepackage{thm-restate}

\usepackage[textsize=tiny]{todonotes}

\newcommand{\norm}[3]{\left\| #1 \right\|_{#2}^{#3}}

\renewcommand{\iota}{\textit{i}}
\renewcommand{\implies}{\; \Rightarrow \;}

\newcommand{\forAll}{{\;\; \forall \;}}

\renewcommand{\hat}[1]{\widehat{#1}} % hat

\newcommand{\wb}[1]{\widebar{#1}}
\newcommand{\E}[1]{\underset{\begin{subarray}{c} #1 \end{subarray}}{\Ebb}}

\newcommand{\TV}[2]{D_\text{TV}(#1\parallel #2)}

\newcommand{\pik}{\pi_{k}}
\newcommand{\pikup}{\pi_{k+1}}
\newcommand{\pol}{\pi_{\theta}}

\newcommand{\polk}{\pi_{\theta_k}}
\newcommand{\polkup}{\pi_{\theta_{k+1}}}
\newcommand{\grad}{\nabla_{\theta}}

\newcommand{\vfunc}{\wb{V}^{\pi}}
\newcommand{\vfuncc}{\wb{V}_{C_{i}}^{\pi}}
\newcommand{\qfunc}{\wb{Q}^{\pi}}
\newcommand{\qfuncc}{\wb{Q}_{C_{i}}^{\pi}}

\newcommand{\vphi}{\wb{V}^{\pi}_{\phi}}
\newcommand{\adv}{\wb{A}^{\pi}}
\newcommand{\advc}{\wb{A}_{C_{i}}^{\pi}}

\newcommand{\jfuncp}{J_{C_{i}}(\pi)}

\newcommand{\advd}{A_{\gamma}^{\pi}}
\newcommand{\vtarg}{\wb{V}^{\text{target}}}

\newcommand{\advk}{\wb{A}^{\pik}}

\newcommand{\dpi}{d_{\pi}}
\newcommand{\dpip}{d_{\pi'}}
\newcommand{\dpid}{d_{\pi,\gamma}}

\newcommand{\dpik}{d_{\pik}}
\newcommand{\dpolk}{d_{\pi_{\theta_k}}}

\newcommand{\Zpi}{Z^{\pi}}
\newcommand{\Mpip}{M^{\pi'}}

\newcommand{\Zpip}{Z^{\pi'}}
\newcommand{\dg}{\text{dg}}

\newcommand{\KL}[2]{D_\text{KL}\left(#1\middle\| #2\right)}

\newcommand{\avKL}[2]{\bar{D}_\text{KL}(#1\parallel #2)}

\newcommand{\Ebb}{{\mathbb E}}

\newcommand{\Nbb}{{\mathbb N}}

\newcommand{\Rbb}{{\mathbb R}}

\DeclareMathOperator*{\argmax}{arg\,max}

\makeatletter
\newcommand{\revisionhistory}[1]{%
\@ifundefined{showrevisionhistory}{\relax}{%
{#1}%
}%
}
\makeatother

\icmltitlerunning{ACPO}
\begin{document}

\twocolumn[
\icmltitle{ACPO: A Policy Optimization Algorithm for Average MDPs with Constraints}


\icmlsetsymbol{equal}{*}

\begin{icmlauthorlist}
\icmlauthor{Akhil Agnihotri}{usc}
\icmlauthor{Rahul Jain}{usc}
\icmlauthor{Haipeng Luo}{usc}

\end{icmlauthorlist}

\icmlaffiliation{usc}{University of Southern California, Los Angeles, CA, USA. RJ is also affiliated with Google DeepMind}

\icmlcorrespondingauthor{Akhil Agnihotri}{agnihotri.akhil@gmail.com}

\icmlkeywords{Machine Learning, ICML}
\vskip 0.3in
]



\printAffiliationsAndNotice{} % otherwise use the standard text.

\begin{abstract}
Reinforcement Learning (RL) for constrained MDPs (CMDPs) is an increasingly important problem for various applications. Often, the average criterion is more suitable than the discounted criterion. Yet, RL for average-CMDPs (ACMDPs) remains a challenging problem. Algorithms designed for discounted constrained RL problems often do not perform well for the average CMDP setting. In this paper, we introduce a new policy optimization with function approximation algorithm for constrained MDPs with the average criterion. The Average-Constrained Policy Optimization (ACPO) algorithm is inspired by trust region-based policy optimization algorithms. We develop basic sensitivity theory for average CMDPs, and then use the corresponding bounds in the design of the algorithm. We provide theoretical guarantees on its performance, and through extensive experimental work in various challenging OpenAI Gym environments, show its superior empirical performance when compared to other state-of-the-art algorithms adapted for the ACMDPs.
\end{abstract}

\section{Introduction}\label{sec:intro}

Over the last decade, we have seen an enormous impact of RL techniques on a variety of problems, from mastering complex games like Go  \cite{silver2017mastering} and StarCraft \cite{vinyals2019grandmaster} to robotic control \cite{levine2016end,akkaya2019solving, aractingi2023controlling}. Many of these have used RL-policy  optimization algorithms such as \citet{schulman2017proximal} for discounted MDPs (DMDPs). These have come in handy even in generative AI, e.g., training large language models (LLMs) \cite{achiam2023gpt}. However, applications often need satisfaction of some \textit{constraints}, e.g., physical safety of mobile robots \cite{Hu_2022_CVPR}, safe language, image or multi-modal output generation. Furthermore, the \textit{average criterion} when long-term rewards and safety are of consideration is more suitable. Using discounted cost formulations (as a proxy for safety) incentivizes policy optimization algorithms to search for policies that are short-term safe but not long-term  because of future-discounting. 

The planning problem for MDPs with constraints is often formulated as a Constrained MDP (CMDP) model \cite{manne1960linear,hordijk1979linear,altman1999constrained}. Unfortunately, CMDP models do not satisfy Bellman’s principle of optimality, and hence dynamic programming (DP)-style algorithms cannot be developed for the setting. Instead, an alternative approach called the convex analytic approach \cite{borkar1988convex, altman1999constrained} is used by way of introducing occupation measures that leads to optimization formulations. This can be done for both discounted (DCMDPs) and average-criterion (ACMDPs) constrained MDPs.

Theory and algorithms for RL deal with  settings when the MDP model is unknown. While DP-inspired RL algorithms such as DQN, when combined with deep learning architectures for function approximation work remarkably effectively \cite{mnih2015human}, policy optimization algorithms such as TRPO \cite{schulman2015trust}, PPO \cite{schulman2017proximal} have proven even more effective in solving high dimensional problems. Since the discounted criterion is sometimes not suitable, policy optimization algorithms such as ATRPO \cite{zhang2021average, wan2021learning, liao2022batch} have been developed for AMDPs. Furthermore, as already mentioned, certain RL applications have multiple objectives, one of which is to be optimized and the rest constrained. Thus, the Constrained Policy Optimization (CPO) algorithm \cite{achiam2017constrained} was  introduced for infinite-horizon DCMDP problems. Unfortunately, as motivated above, not all such applications fit the discounted-criterion formulation: there are settings, for example where there may be safety requirements when the average-CMDP model is a better fit. No scalable RL algorithms are currently available for such settings. 

We note that the RL problem is usually harder than the corresponding planning problem;   average-MDPs are  more challenging than discounted MDPs; and constrained MDPs are more challenging than unconstrained ones. In this paper, we present the first practical algorithm for policy optimization-based RL algorithm for average-constrained MDPs. We propose ACPO, a policy optimization algorithm for an average-CMDP  with deep learning for function approximation. Our approach is motivated by theoretical guarantees that bound the difference between the average long-run rewards or costs of different policies. It draws inspiration from CPO \cite{achiam2017constrained} (see also \citet{tessler2018reward} for DCMDPs), which uses a policy improvement theorem for the discounted setting based on the trust region methods of \citet{schulman2015trust}. Unfortunately, this result trivializes for the average setting and hence can't be used. Instead, we derive a new bound that depends on the worst-case level of ``mixture" of the irreducible Markov chain associated with a policy. Our proposed algorithm, ACPO is based on these theoretical developments. For experimental evaluation, we use several OpenAI Gym environments from \citet{todorov2012mujoco}, train large neural network policies and demonstrate the effectiveness and superior performance of the ACPO algorithm as compared to others.

\noindent\textbf{Main Contributions and Novelty.} 

\textit{Algorithmic:} We introduce the first practical policy optimization-based RL algorithm for average-constrained MDPs with new and tight performance bounds and violation guarantees. The algorithm draws inspiration from CPO (for discounted-CMDPs) and ATRPO (for average-MDPs) but is not a straightforward extension of either. One may posit that setting the discount factor $\gamma = 1$ in CPO for the discounted setting may suffice but that does not perform well on average-CMDPs even with a large discount factor. Further, constraint violation and policy degradation bounds of CPO do not hold in the average setting and hence we develop novel bounds (in Corollary \ref{cor:thm:avg_constraint_imp}). In fact,  the advantage function estimation routine in our algorithm (line 4 and 6 in Algorithm \ref{alg:practical_acpo}) is also different from that in CPO, since the discounted-setting procedure cannot be used for the average setting (see Appendix \ref{appendix:approx_acpo}): We first approximate the average-reward bias and then use a one-step TD backup to estimate the action-bias function. Furthermore, policy optimization algorithms for the average case \cite{zhang2020average, wan2021learning, liao2022batch} cannot incorporate  constraints. We enable this by introducing sublevel sets of cost constraints. We also introduce an approximate but novel line search procedure that improves the empirical performance of our algorithm, an idea that may help improve performance of other  policy optimization algorithms such as  PPO.\\
\textit{Technical:} Since ACPO is a trust region method, one can expect some overlap in analysis techniques with other similar algorithms.  Nevertheless, our analysis has several novel elements: Lemma \ref{lemma:d_and_pi}, where we use eigenvalues of the transition matrix to relate total variation  of stationary distributions with that of the policies, and in Lemma \ref{lemma:infeasible_kl_bound}, we use the sublevel sets of constraints and projection inequality of Bregman divergence. Furthermore, several important results from CPO and ATRPO papers cannot be applied to the analysis of our algorithm. \\
\textit{Empirical:} We evaluate the empirical performance of ACPO in the OpenAI Gym (Mujoco) environments, a standard benchmark. We find that ACPO outperforms all state-of-the-art Deep RL algorithms  such as CPO in \cite{achiam2017constrained}, PCPO in \cite{yang2020projection}, PPO in \cite{schulman2017proximal}, BVF-PPO in \cite{pmlr-v119-satija20a} and ATRPO in \cite{zhang2021average}. We use a large discount factor if the algorithm is not for the average setting, and a Lagrangian objective if it is not for the constrained setting, and in some cases both.\\
\textit{Significance:} ACPO is the first practical  trust region-style policy optimization algorithm for ACMDPs with excellent empirical performance. ACMDPs are important models because they allow incorporation of long term safety constraints, which are important not only in the context of safe robot learning and control, but also safety-constrained RLHF fine-tuning and inference for LLMs \cite{moskovitz2023confronting} and Diffusion models as well. In the absence of suitable policy optimization algorithms for ACMDPs, researchers have resorted to using adaptations of PPO, etc.

\noindent\textbf{Related Work.}
Learning constraint-satisfaction policies has been explored in the Deep RL literature in \cite{9030307, garcia2015comprehensive}. This can either be  (1) through expert annotations and demonstrations as in \cite{rajeswaran2017learning,gao2018reinforcement} or, (2) by exploration with constraint satisfaction as in \cite{achiam2017constrained,tessler2018reward}. While the former approach is not scalable since it requires human interventions, current state-of-the-art algorithms for the latter are not applicable to the average reward setting. 

Previous work on RL with the average reward criterion has mostly attempted to extend stochastic approximation schemes for the tabular setting, such as Q-learning in \cite{abounadi2001learning,wan2021learning}, to the non-tabular setting with function approximation in \cite{wei2021learning, zhang2020average}. \cite{pmlr-v162-chen22i} deals with online learning in a constrained MDP setting, but their aim is regret minimization or exploration, both in tabular settings. 
% \cite{zhang2020average} provide bounds on the performance of a trust region algorithm for the  the average reward setting but do not incorporate constraints. 
We are inspired by the work of \cite{zhang2021average} to develop techniques required to derive the policy degradation and constraint violation bounds in Section \ref{sec:algo}.

%These bounds are related to the Kemeny's constant in \cite{kemeny1960finite}, which is related to the mixing time of a Markov chain. However, these algorithms do not incorporate constraints during learning.

%Compared with these works, we believe our results and policy search algorithm are the first average reward CMDP version that is both based in theory and practically implementable for large state-action spaces with arbitrary policy classes, as shown by our experiments. 

The more recent works of \cite{Bhatnagar2012} and \cite{calvo2023state} also fail to address our problem setting as 
the former test on a 2x4 queueing network with maximum state space of 128, while the latter test on a grid of size 10x10 (maximum states of 100). In addition to that, the way they incorporate constraints during training is just via a Lagrangian formulation. In our paper we show that simply doing this (in the case of PPO and ATRPO for example) leads to much inferior performance to ACPO, which can outperform current state-of-the-art algorithms in state spaces of upto $10^{96}$.

\section{Preliminaries}\label{sec:prelims}

A Markov decision process (MDP) is a tuple, ($S,A,r,P,\mu$), where $S$ is the set of states, $A$ is the set of actions, $r : S \times A \times S \to \Rbb$ is the reward function, $P : S \times A \times S \to [0,1]$ is the transition probability function such that $P(s'|s,a)$ is the probability of transitioning to state $s'$ from state $s$ by taking action $a$, and $\mu : S \to [0,1]$ is the initial state distribution. A stationary policy $\pi : S \to \Delta(A)$ is a mapping from states to probability distributions over the actions, with $\pi(a|s)$ denoting the probability of selecting action $a$ in state $s$, and $\Delta(A)$ is the probability simplex over the action space $A$. We denote the set of all stationary policies by $\Pi$. For the average setting, we will make the standard assumption that the MDP is \textit{ergodic} and is \textit{unichain}.

In reinforcement learning, we aim to select a policy $\pi$ which maximizes a performance measure, $J(\pi)$, which, for continuous control tasks is either the discounted reward criterion or the average reward approach. Below, we briefly discuss both formulations.

\subsection{Discounted criterion}

For a given discount factor $\gamma\in (0,1)$, the discounted reward objective is defined as 

\begin{equation*}
\resizebox{.61\linewidth}{!}{$
\begin{aligned}
    J_{\gamma}(\pi) &:=  \E{\tau\sim\pi} \left[\sum_{t=0}^{\infty}\gamma^t r(s_t,a_t,s_{t+1}) \right] \\ &= \frac{1}{1-\gamma}\E{s\sim\dpid\\ a\sim\pi \\ s'\sim P(\cdot | s,a)}[r(s,a,s')]
\end{aligned} 
$}
\end{equation*}

where $\tau$ refers to a sample trajectory of $(s_0, a_0, s_1, \cdots)$ generated when following a policy $\pi$, that is, $a_t \sim \pi(\cdot | s_t)$ and $s_{t+1} \sim P(\cdot | s_t, a_t) \,$; $d_{\pi, \gamma}$ is the \textit{discounted occupation measure} that is defined by $d_{\pi, \gamma} (s) = (1-\gamma) \sum_{t=0}^{\infty} \gamma^t \underset{\tau\sim\pi}{P}(s_t = s)$, which essentially refers to the discounted fraction of time spent in state $s$ while following policy $\pi$. 

% The \textit{discounted value function} is then defined as $$\vfuncd(s):= \E{\tau\sim\pi}\left[\sum_{t=0}^{\infty} \gamma^t r(s_t,a_t,s_{t+1})\bigg| s_0=s\right]$$ and \textit{discounted action-value function} 
% \begin{align*}
%     \qfuncd(s,a):= \E{\tau\sim\pi}\bigg[\sum_{t=0}^{\infty} \gamma^t r(s_t,a_t,& s_{t+1}) \, \bigg| \, s_0=s, a_0=a \bigg].
% \end{align*} 
% Finally, the \textit{discounted advantage function} is defined as $\advd(s,a) := \qfuncd(s,a)-\vfuncd(s)$.

\subsection{Average criterion}

The average-reward objective is given by:
\begin{equation}\label{eq:AvgR_obj}
\begin{split}
    J(\pi) &:= \lim_{N\to\infty}\frac{1}{N}\E{\tau\sim\pi}\left[\sum_{t=0}^{N-1} r(s_t,a_t, s_{t+1})\right] \\ &= \E{s\sim\dpi\\ a\sim\pi(\cdot|s) \\ s' \sim P(\cdot | s,a)}[r(s,a,s')],
\end{split}
\end{equation}
where $\dpi(s):=\lim_{N\to\infty}\frac{1}{N} \sum_{t=0}^{N-1} P_{\tau\sim\pi}(s_t=s)$ is the \textit{stationary state distribution} under policy $\pi$. The limits in $J(\pi)$ and $\dpi(s)$ are guaranteed to exist under our ergodic assumption. Since the MDP is aperiodic, it can also be shown that $\dpi(s)=\lim_{t\to\infty}P_{\tau\sim\pi}(s_t=s)$. Since we have $\lim_{\gamma\to 1} \dpid(s)\to\dpi(s), \forall s$, it can be shown that $\lim_{\gamma\to 1} (1-\gamma) J_{\gamma}(\pi) =J(\pi)$.

In the average setting, we seek to keep the estimate of the state value function unbiased and hence, introduce the \textit{average-reward bias function} as
\begin{equation*}
    \vfunc(s):= \E{\tau\sim\pi}\left[\sum_{t=0}^{\infty} (r(s_t,a_t,s_{t+1}) - J(\pi)) \; \bigg| \; s_{0}=s\right]
\end{equation*}
and the \textit{average-reward action-bias function} as
\begin{align*}
\qfunc(s,a):= \E{\tau\sim\pi}\bigg[\sum_{t=0}^{\infty} (r(s_t,a_t,s_{t+1}) - & J(\pi)) \; \bigg| \; \underset{a_{0}=a}{s_{0}=s,} \bigg].
\end{align*}
Finally, define the \textit{average-reward advantage function} as $ \adv(s,a) := \qfunc(s,a)-\vfunc(s)$.

 %We now focus our attention on describing the problem setting.

%%%%%%%%%%%%%%%%%%%%%%%%%%%%%%%%%%%%%%%%%%%%%%%%%%%%%%%%%%%%%%%%%%%%%%%%

\subsection{Constrained MDPs}

A constrained Markov decision process (CMDP) is an MDP augmented with constraints that restrict the set of allowable policies for that MDP.  Specifically, we augment the MDP with a set $C$ of auxiliary cost functions, $C_1, \cdots , C_m$ (with each function $C_i : S \times A \times S \to \Rbb$ mapping transition tuples to costs, just like the reward function), and bounds $l_1, \cdots , l_m$. Similar to the value functions being defined for the average reward criterion, we define the average cost objective with respect to the cost function $C_{i}$ as
\begin{equation}
\begin{split}
    J_{C_{i}}(\pi) &:= \lim_{N\to\infty}\frac{1}{N}\E{\tau\sim\pi}\left[\sum_{t=0}^{N-1} C_{i}(s_t,a_t, s_{t+1})\right] \\ 
    &= \E{s\sim\dpi\\ a\sim\pi \\ s' \sim P(\cdot | s,a)}[C_{i}(s,a,s')].
\end{split}
\end{equation}
where $J_{C_i}$ will  be referred to as the \textit{average cost} for constraint $C_{i}$. The set of feasible stationary policies for a CMDP then is given by
$
\Pi_{C} := \left\{\pi \in \Pi \; : \; \jfuncp \leq l_i, \forAll i \in \{1, \cdots , M\} \right\}
$.
The goal is to find a policy $\pi^{\star}$ such that 
$\pi^{\star} \in \arg\max_{ \pi \in \Pi_{C} } J(\pi).$

However, finding an exact $\pi^{\star}$ is infeasible for large-scale problems.
Instead,  we aim to derive an iterative policy improvement algorithm that given a current policy, improves upon it by approximately maximizing the increase in the reward, while not violating the constraints by too much and not being too different from the current policy.

%Starting with an initial random policy, we repeat the following procedure: approximately calculate the average reward advantage function of the current policy and then obtain the next policy by maximize a lower bound related to this advantage function to obtain a new policy. If the new policy already satisfies the constraints, then it is guaranteed to have higher reward compared to the last policy. Otherwise, we use projection based backtracking to refine the policy, which can only lead to a certain amount of performance degradation.

% The choice of optimizing only over stationary policies is justified: it has been shown that the set of all optimal policies for a CMDP includes stationary policies, under mild technical conditions.

Lastly, %since we only focus on the average rewards setting, 
analogous to $\vfunc$, $\qfunc$, and $\adv$, we define similar  quantities for the cost functions  $C_i(\cdot)$, and denote them by $\vfuncc$, $\qfuncc$, and $\advc$.

\subsection{Policy Improvement for discounted CMDPs}

% When we have large state and action spaces, solving for the exact optimal policy is intractable \cite{Sutton1998} and hence, we focus only on parameterized policy classes $\Pi_{\theta}$, where $\theta$ are the parameters of the policy. 

% To develop theoretical guarantees, as we will do next, we make the following assumptions to ensure convergence to a (optimal) policy within $\Pi_{C}$:
% \begin{assumption}\label{bounded_value} \todo{no need}
%     The values $V^\pi(s)$ and hence $J(\pi)$ are bounded for all policies $\pi \in \Pi_{C}$.
% \end{assumption}
% \begin{assumption}\label{existance_of_policy}
%     Every local minima of $\Bar{J}_{C_{i}}(\pi)$ is a feasible solution.
% \end{assumption}
% Assumption \ref{existance_of_policy} is the bare minimum requirement we need to ensure any kind of convergence of gradient based optimization algorithms. Other assumptions, such as convexity of $J(\pi)$ of course ensure convergence to the optimal policy, but it rarely holds in practice.

In many on-policy constrained RL problems, we improve policies iteratively by maximizing a predefined function within a local region of the current best policy as in \cite{tessler2018reward, achiam2017constrained, yang2020projection, song2020v}.
% Quantitatively, for the discounted setting, at iteration $k$ we find a policy $\pi_{k+1}$ by maximizing $J_{\gamma}(\pi)$ within some region $D(\pikup,\pik)\leq\delta$, where $D$ is some notion of a divergence measure. Different choices of $D$ and $\delta$ lead to different algorithms with variable step-sizes \cite{peters2008reinforcement}. 
\cite{achiam2017constrained} derived a policy improvement bound for the discounted CMDP setting as:

\begin{equation}\label{eq:cpo_improve}
\begin{aligned}
&J_{\gamma}(\pikup) - J_{\gamma}(\pik) \geq \\ &\frac{1}{1-\gamma} \E{s \sim d^{\pik} \\ a \sim \pikup} \left[ A_{\gamma}^{\pik} (s,a) - \frac{2\gamma \epsilon^{\pikup}}{1-\gamma}  D_{TV} (\pikup||\pik)[s] \right], &
\end{aligned}
\end{equation}
where $A_{\gamma}^{\pik}$ is the discounted version of the advantage function, $\epsilon^{\pikup} := \max_s | \Ebb_{a \sim \pikup} [A_{\gamma}^{\pik} (s,a) ] |$, and  $D_{TV}(\pikup||\pik)[s] = (1/2)\sum_a \left| \pikup(a|s) - \pik(a|s) \right|$ is the total variational divergence between $\pikup$ and $\pik$ at $s$. These results laid the foundations for on-policy constrained RL algorithms as in\cite{wu2017scalable,vuong2019supervised}.

However, Equation \eqref{eq:cpo_improve} does not generalize to the average setting ($\gamma \to 1$) (see Appendix \ref{proof:cpo_bound_trivial}). In the next section, we will derive a policy improvement bound for the average  case  and present an algorithm based on trust region methods, which will generate almost-monotonically improving iterative policies. Proofs of theorems and lemmas, if not already given, are available in Appendix \ref{sec:appendix}.

\section{ACPO: The Average-Constrained Policy Optimization Algorithm}\label{sec:algo}

In this section, we present the main results of our work. For conciseness, we denote by $\dpi\in\Rbb^{|S|}$ the column vector whose components are $\dpi(s)$ and $P_{\pi}\in\Rbb^{|S|\times|S|}$ to be the state transition probability matrix under policy $\pi$.
%such that $P_{\pi}(s'|s)=\sum_a \pi(a|s) P(s'|s,a)$.

\subsection{Policy Improvement for the Average-CMDP}

Let $\pi'$ be the policy obtained via some update rule from the current policy $\pi$. Analogous to the discounted setting of a CMDP, we would like to characterize the performance difference $J(\pi')-J(\pi)$ by an expression which depends on $\pi$ and some divergence metric between the two policies.

\begin{restatable}{lemma}{policydiff} \cite{zhang2020average}
\label{lemma:policy_diff}
Under the unichain assumption of the underlying Markov chain, for any stochastic policies $\pi$ and $\pi'$:
\begin{equation}\label{eq:policy_diff}
    J(\pi') - J(\pi) =  \E{s\sim \dpip \\ a\sim\pi'}\left[\adv(s,a)\right].
\end{equation}
\end{restatable}

Note that this difference depends on the stationary state distribution obtained from the \textit{new} policy, $\dpip$. This is computationally impractical as we do not have access to this $\dpip$. Fortunately, by use of the  following lemma we can show that if $\dpi$ and $\dpip$ are ``close'' with respect to some metric, we can approximate Eq.  \eqref{eq:policy_diff} using samples from $\dpi$.

\begin{restatable}{lemma}{policyimpd}
\label{lemma:policy_impd}
Under the unichain assumption, for any stochastic policies $\pi$ and $\pi'$ we have:
\begin{equation}
\label{eq:policy_diff_tv}
\begin{split}
    \left| J(\pi') - J(\pi) -\E{s\sim \dpi\\ a\sim\pi'}\left[\adv(s,a)\right] \right| \leq 2\epsilon\TV{\dpip}{\dpi}
\end{split}
\end{equation}
\end{restatable}
, where $\epsilon=\max_s\big|\E{a\sim\pi'}[\adv(s,a)]\big|$.
See Appendix \ref{proof:policy_impd_proof} for proof. Lemma \ref{lemma:policy_impd} implies $J(\pi')\approx J(\pi) +\E{} [\adv(s,a)]$ when $\dpi$ and $\dpip$ are ``close''. Now that we have established this approximation, we need to study the relation of how the actual change in policies affects their corresponding stationary state distributions. For this,  we turn to standard analysis of the underlying Markov chain of the CMDP. 

Under the ergodic assumption, we have that $P_{\pi}$ is irreducible and hence its eigenvalues $\{\lambda_{\pi, i}\}_{i=1}^{|S|}$ are such that $\lambda_{\pi, 1}=1$ and $\lambda_{\pi, i \neq 1} < 1$. For our analysis, we define $\sigma^{\pi} = \max_{i\neq 1} \, (1-\lambda_{\pi, i})^{-1/2}$, and from \cite{levene2002kemeny} and \cite{doyle2009kemeny}, we connect $\{\lambda_{\pi, i}\}_{i=1}^{|S|}$ to the sensitivity of the stationary distributions to changes in the policy using the result below. 

\begin{restatable}{lemma}{dandpi}
\label{lemma:d_and_pi}
Under the ergodic assumption, the divergence between the stationary distributions $\dpi$ and $\dpip$ is upper bounded as:
\begin{equation}
    \TV{\dpip}{\dpi} \leq \sigma^{\star} \E{s\sim \dpi}[\TV{\pi'}{\pi}[s]],
\end{equation}
\end{restatable}
, where $\sigma^{\star} = \max_{\pi} \sigma^{\pi}$. See Appendix \ref{proof:d_and_pi} for proof. This bound is tighter and easier to compute than the one given by \cite{zhang2021average}, which replaces $\sigma^{\star}$ by $\kappa^{\star} = \max_{\pi} \kappa^{\pi}$, where $\kappa^{\pi}$ is known as  \textit{Kemeny's constant} from \cite{kemeny1960finite}. It is interpreted as the expected number of steps to get to any goal state, where the expectation is taken with respect to the stationary-distribution of those states. 
% From \cite{hunter2006mixing}, we see that if the mixing time and $\kappa^{\pi}$ are relatively small, then so is $\sigma^{\star}$, which shows that for small changes in policy, the impact on the stationary distributions is small. 

Combining the bounds in Lemma \ref{lemma:policy_impd} and Lemma \ref{lemma:d_and_pi} gives us the following result:

\begin{proposition}\label{prop:avg_policy_imp}
Under the ergodic assumption, the following bounds hold for any stochastic policies $\pi$ and $\pi'$:
\begin{equation}\label{eq:avg_policy_imp}
 L_{\pi}^{-}(\pi')\leq J(\pi') - J(\pi)\leq L_{\pi}^{+}(\pi') 
\end{equation}
where
\begin{align*}
    L_{\pi}^{\pm}(\pi') &= \E{s\sim \dpi\\ a\sim\pi'}\left[\adv(s,a)\right] \pm 2 \nu \E{s\sim \dpi}[\TV{\pi'}{\pi}[s]]  \\ \text{and} \quad \nu &= \sigma^{\star} \max_{s} \big| \E{a\sim\pi'}[\adv(s,a)] \big|.
\end{align*}

\end{proposition}

It is interesting to compare the inequalities of Equation \eqref{eq:avg_policy_imp} to Equation \eqref{eq:policy_diff}. The term $\E{}[\adv(s,a)]$ in Prop. \ref{prop:avg_policy_imp} is somewhat of a \textit{surrogate} approximation to $J(\pi') - J(\pi)$, in the sense that it uses $d_{\pi}$ instead of $d_{\pi'}$. As discussed before, we do not have access to $d_{\pi'}$ since the trajectories of the new policy are not available unless the policy itself is updated. This surrogate is a first order approximation to $J(\pi') - J(\pi)$ in the parameters of $\pi'$ in a neighborhood around $\pi$ as in \cite{kakade2002approximately}. Hence, Eq.  \eqref{eq:avg_policy_imp} can be viewed as bounding the worst-case approximation error.

Extending this discussion to the cost function of our CMDP, similar expressions follow immediately.

\begin{corollary}
\label{cor:thm:avg_constraint_imp}
For any policies $\pi', \pi$, and any cost function $C_i$, the following bound holds:

\begin{equation}\label{eq:avg_constraint_imp}
 M_{\pi}^{-}(\pi')\leq J_{C_i} (\pi') - J_{C_i} (\pi) \leq M_{\pi}^{+}(\pi') 
\end{equation}
where
\begin{align*}
    M_{\pi}^{\pm}(\pi') &= \E{s\sim \dpi\\ a\sim\pi'}\left[\advc(s,a)\right] \pm 2 \nu_{C_{i}} \E{s\sim \dpi}[\TV{\pi'}{\pi}[s]] \\ \text{and} \quad \nu_{C_{i}} &= \sigma^{\star} \max_{s} \big| \E{a\sim\pi'}[\advc(s,a)]\big|.
\end{align*}
\end{corollary}

Until now, we have been dealing with bounds given with regards to the TV divergence of the policies. However, in practice, bounds with respect to the KL divergence of policies is more commonly used as in \cite{schulman2015trust, schulman2016high, ma2021average}. 
From Pinsker’s and Jensen's inequalities, we have that 
\begin{equation}\label{eq:tv-kl}
\begin{aligned}
  \E{s\sim\dpi}\big[\TV{\pi'}{\pi}[s]\big] 
  &\leq \sqrt{\E{s\sim\dpi}\big[\KL{\pi'}{\pi}][s]\big]/2}.
\end{aligned}
\end{equation}

%Since TV divergence is upper bounded by KL divergence, 
We can thus use Eq. \eqref{eq:tv-kl} in the bounds of Proposition  \ref{prop:avg_policy_imp} and Corollary \ref{cor:thm:avg_constraint_imp} to make policy improvement guarantees, i.e., if we find updates such that $\pikup \in \argmax_{\pi} L_{\pik}^{-}(\pi)$, then we will have monotonically increasing policies as, at iteration $k$, $\E{s\sim\dpik,  a\sim\pi}[\advk(s,a)]=0$, $\E{s\sim\dpik}[\KL{\pi}{\pik}[s]]=0$ for $\pi=\pik$, implying that $J(\pikup)-J(\pik)\geq 0$. However, this sequence does not guarantee constraint satisfaction at each iteration, so we now turn to trust region methods to incorporate constraints, do policy improvement and provide safety guarantees.

\subsection{Trust Region Based Approach}

For large or continuous state and action CMDPs, solving for the exact optimal policy is impractical. However, \textit{trust region}-based policy optimization algorithms have proven to be effective for solving such problems as in \cite{schulman2015trust, schulman2016high, schulman2017proximal, achiam2017advanced}.
For these approaches, we usually consider some parameterized policy class $\Pi_{\Theta} = \{\pi_{\theta}: \theta \in \Theta \}$ for tractibility. In addition, for CMDPs, we also require the policy iterates to be feasible, so instead of optimizing just over $\Pi_{\Theta}$, we optimize over $\Pi_{\Theta} \cap \Pi_C$. However, it is much easier to solve the above problem if we introduce hard constraints, rather than limiting the set to $\Pi_{\Theta} \cap \Pi_C$. Therefore, we now introduce the ACPO algorithm, which is inspired by the trust region formulations above as the following optimization problem:
\begin{equation}\label{eq:acpo_trust}
\begin{aligned}
  \underset{\pi\in\Pi_{\Theta}}{\text{maximize}} & \quad
\E{s\sim\dpolk\\ a\sim\pi}[\wb{A}^{\polk}(s,a)] \\
\text{s.t.} &\qquad J_{C_i} (\polk) + \E{s\sim\dpolk\\ a\sim\pi}[\wb{A}_{C_{i}}^{\polk}(s,a)] \leq l_i, \;\;\; \forAll i \\
& \qquad \avKL{\pi}{\polk} \leq \delta
\end{aligned}
\end{equation}

where $\avKL{\pi}{\polk} := \E{s\sim\dpolk}[\KL{\pi}{\polk}[s]]$, $\wb{A}^{\polk}(s,a)$ is the average advantage function defined earlier, and $\delta > 0$ is a step size. We use this form of updates as it is an approximation to the lower bound given in Proposition \ref{prop:avg_policy_imp} and the upper bound given in Corollary \ref{cor:thm:avg_constraint_imp}.

In most cases, the trust region threshold for formulations like Eq. \eqref{eq:acpo_trust} are heuristically motivated. We now show that it is quantitatively motivated and comes with a worst case performance degradation and constraint violation. Proof is in Appendix \ref{proof:trust_proof}.

\begin{restatable}{theorem}{trustdegradationviolation}
\label{th:trust_degradation_violation}
Let $\polkup$ be the optimal solution to Eq.  \eqref{eq:acpo_trust} for some $\polk\in\Pi_{\Theta}$. Then, we have
\begin{align}
J(\polkup)-J(\polk)\geq&-\sqrt{2(\delta+V_{max})}\nu^{\polkup} \\
\text{and} \;\; J_{C_{i}}(\polkup)\leq l_{i} & +\sqrt{2(\delta+V_{max})}\nu^{\polkup}_{C_{i}} \forAll i,
\end{align}
where $\nu^{\polkup}=\sigma^{\polkup}\max_{s}\big|\E{a\sim\polkup}[\wb{A}^{\polk}(s,a)]\big|$, $\nu^{\polkup}_{C_{i}} = \sigma^{\polkup} \max_{i,s}\big|\E{a\sim\polkup}[\wb{A}_{C_{i}}^{\polk}(s,a)]\big|$, $V_{max} = \max_{i} \beta_{i}^{2}$, ~\text{and}~ $\beta_{i} = [J_{C_{i}}(\polk)-l_{i}]_{+}$.
\end{restatable}

 %This result shows that we can develop ACPO as a trust region method so that it inherits these guarantees.
 \textbf{Remark 3.7.} Note that if the constraints are ignored (by setting $V_{max}=0$), then this bound is tighter than given in \cite{zhang2021average} for the unconstrained average-reward setting.
 
However, the update rule of Eq.  \eqref{eq:acpo_trust} is difficult to implement in practice as it takes steps that are too small, which degrades convergence. In addition, it requires the exact knowledge of $\wb{A}^{\polk}(s,a)$ which is computationally infeasible for large-scale problems. In the next section, we will introduce a specific sampling-based practical algorithm to alleviate these concerns.

\section{Practical Implementation of ACPO}
\label{sec:acpo_implementation}

In this section, we introduce a practical version of the ACPO algorithm with a principle recovery method. With a small step size $\delta$, we can approximate the reward function and constraints with a first order expansion, and approximate the KL divergence constraint with a second order expansion. This gives us a new optimization problem which can be solved exactly using Lagrangian duality.

\subsection{An Implementation of ACPO}

Since we are working with a parameterized class, we shall now overload notation to use $\theta_{k}$ as the policy at iteration $k$, i.e., $\theta_{k} \equiv \polk$. In addition, we use $g$ to denote the gradient of the advantage function objective, $a_{i}$ to denote the gradient of the advantage function of  the cost $C_{i}$, $H$ as the Hessian of the KL-divergence. Formally,
{\small
\begin{align*}
    g &:= \nabla_\theta\E{\substack{s\sim d_{\theta_k} \\ a\sim \theta}}[\wb{A}^{\theta_k}(s,a)] \Bigr\vert_{\theta=\theta_k}, \\ 
    a_{i} &:= \nabla_\theta\E{\substack{s\sim d_{\theta_k} \\ a\sim \theta}}[\wb{A}^{\theta_k}_{C_{i}}(s,a)] \Bigr\vert_{\theta=\theta_k}, \\ 
    %H_{x,y} := \frac{\partial^2 \E{s\sim d_{\theta_k}}\big[\KL{\theta}{\theta_k})[s]\big]}{\partial \theta_x\partial \theta_y} 
    H &:= \nabla^2_\theta \E{s\sim d_{\theta_k}}\big[\KL{\theta}{\theta_k})[s]\big] \Bigr\vert_{\theta=\theta_k}.
\end{align*}
}
In addition, let $c_i := J_{C_i}(\theta_k) - l_i$. The approximation to the problem in Eq.  \eqref{eq:acpo_trust} is: 
\begin{equation}
\label{eq:acpo_approx}
\begin{aligned}
\max_{\theta} \;\;\; &  g^T (\theta - \theta_k) & \\
\text{s.t.} & \quad c_i + a_i^T (\theta - \theta_k) \leq 0, \forAll i \\ \text{and,}& \quad \tfrac{1}{2} (\theta - \theta_k)^T H (\theta - \theta_k) \leq \delta.
\end{aligned}
\end{equation}
This is a convex optimization problem in which strong duality holds, and hence it can be solved using a Lagrangian method. The update rule for the dual problem then takes the form
\begin{equation}
\theta_{k+1} = \theta_{k} + \frac{1}{\lambda^{\star}} H^{-1} \left(g - A \mu^{\star}\right). 
\label{eq:acpo_dual_update}
\end{equation}

where $\lambda^{\star}$ and $\mu^{\star}$ are the Lagrange multipliers satisfying the dual
\begin{equation}
\begin{aligned}
\max_{\begin{subarray}{c} \lambda \geq 0 \\ \mu \succeq 0\end{subarray}} \frac{-1}{2\lambda} \left( g^T H^{-1} g - 2 r^T \mu + \mu^T S \mu\right) + \mu^T c - \frac{\lambda \delta}{2},
\end{aligned} 
\label{eq:acpo_dual}
\end{equation}

with $r := g^T H^{-1} A$, $S := A^T H^{-1} A$, $A := [a_{1} , \cdots , a_{m}]$, and $c := [c_{1}, \cdots , c_{m}]^T$.

\subsection{Feasibility and Recovery}
\label{sec:recovery}

The approximation regime described in Eq.  \eqref{eq:acpo_approx} requires $H$ to be invertible. For large parametric policies, $H$ is computed using the conjugate gradient method as in \cite{schulman2015trust}. However, in practice, using this approximation along with the associated statistical sampling errors, there might be potential violations of the approximate constraints leading to infeasible policies. 

To rectify this, for the case where we only have one constraint, one can recover a feasible policy by applying a recovery step inspired by the TRPO update on the cost surrogate as:
\begin{equation}
\resizebox{.9\linewidth}{!}{$
\begin{aligned}
\theta_{k+1/2} = \theta_{k} - \sqrt{2\delta} \bigg[t \cdot \frac{H^{-1} a}{\sqrt{a^T H^{-1} a}} + (1-t) \cdot \frac{H^{-1} g}{\sqrt{g^T H^{-1} g}} \bigg]
\end{aligned}
$}
\label{eq:acpo_recovery}
\end{equation}

where $t \in [0,1]$. 
Contrasting with the policy recovery update of \cite{achiam2017constrained} which only uses the cost advantage function gradient $a$, we introduce the reward advantage function gradient $g$ as well. This choice is to ensure recovery while simultaneously balancing the ``regret'' of not choosing the best (in terms of the objective value) policy $\pik$. In other words, we wish to find a  policy $\pi_{k+1/2}$ as close to $\pik$ in terms of their objective function values. We follow up this step with a simple linesearch to find feasible $\pikup$. Based on this, Algorithm \ref{alg:practical_acpo} provides a basic outline of ACPO. For more details of the algorithm, see Appendix \ref{appendix:approx_acpo}.

% To find estimates of the advantage functions in Line 4 of Algorithm \ref{alg:practical_acpo}, we first get the action probabilities with $\mathcal{O}(|P|)$ cost to calculate the KL divergence, we, where $P$ is the parameter set of the neural policy. To further calculate the Hessian w.r.t. to the policy parameters, we incur $\mathcal{O}(|A|*|P|)$ cost, where $A$ is the action space. Similar complexity results follow for other variables. 

\begin{algorithm}[!t]
   \caption{Average-Constrained Policy Optimization (ACPO)}
   \label{alg:practical_acpo}
\begin{algorithmic}[1]
   \STATE {\bfseries Input:} Initial random policy $\pi_0 \in \Pi_{\theta}$
	 \FOR{$k = 0,1,2,...,K$} 
	 \STATE Sample a set of trajectories $\Omega$ using $\pi_k = \polk$
	 \STATE Find estimates of $g, a, H, c$ using $\Omega$ 
	 \IF{a feasible solution to Equation \eqref{eq:acpo_approx} exists}
	 	\STATE Solve dual problem in Equation \eqref{eq:acpo_dual} for $\lambda^{\star}_k, \mu^{\star}_k$ 
	 	\STATE Find policy update $\pikup$ with Equation \eqref{eq:acpo_dual_update}
	 \ELSE
	 	\STATE Find recovery policy $\pi_{k+1/2}$ with Equation \eqref{eq:acpo_recovery}
	 	\STATE Obtain $\pikup$ by linesearch till approximate constraint satisfaction of Equation \eqref{eq:acpo_approx}
	 \ENDIF
	 
	\ENDFOR
\end{algorithmic}
\end{algorithm}

%    \item     Although, both CPO and ACPO are designed for infinite horizon settings, in practice we have to truncate the horizon up to some finite time $T$. We use a much larger $T$ for ACPO than for CPO because discounting under-emphasizes much more the rewards obtained later on than averaging. 
    %ACPO is an algorithm designed for infinite horizon, average reward setting, which uses thousand times more truncation value $K$ on the horizon than CPO. 
%\end{enumerate}

\begin{figure*}[t]
    \hspace{0.25cm} Average Rewards: \newline
    % {
    %     \includegraphics[width=0.18\textwidth]{images/point_circle_rewards.pdf}
    %     \label{fig:point_circle_rewards}
    % }
    % {
    %     \includegraphics[width=0.18\textwidth]{images/point_gather_rewards.pdf}
    %     \label{fig:point_gather_rewards}
    % }
    % {
    %     \includegraphics[width=0.23\textwidth]{images/ant_circle_rewards.pdf}
    %     \label{fig:ant_circle_rewards}
    % }
    {
        \includegraphics[width=0.32\textwidth]{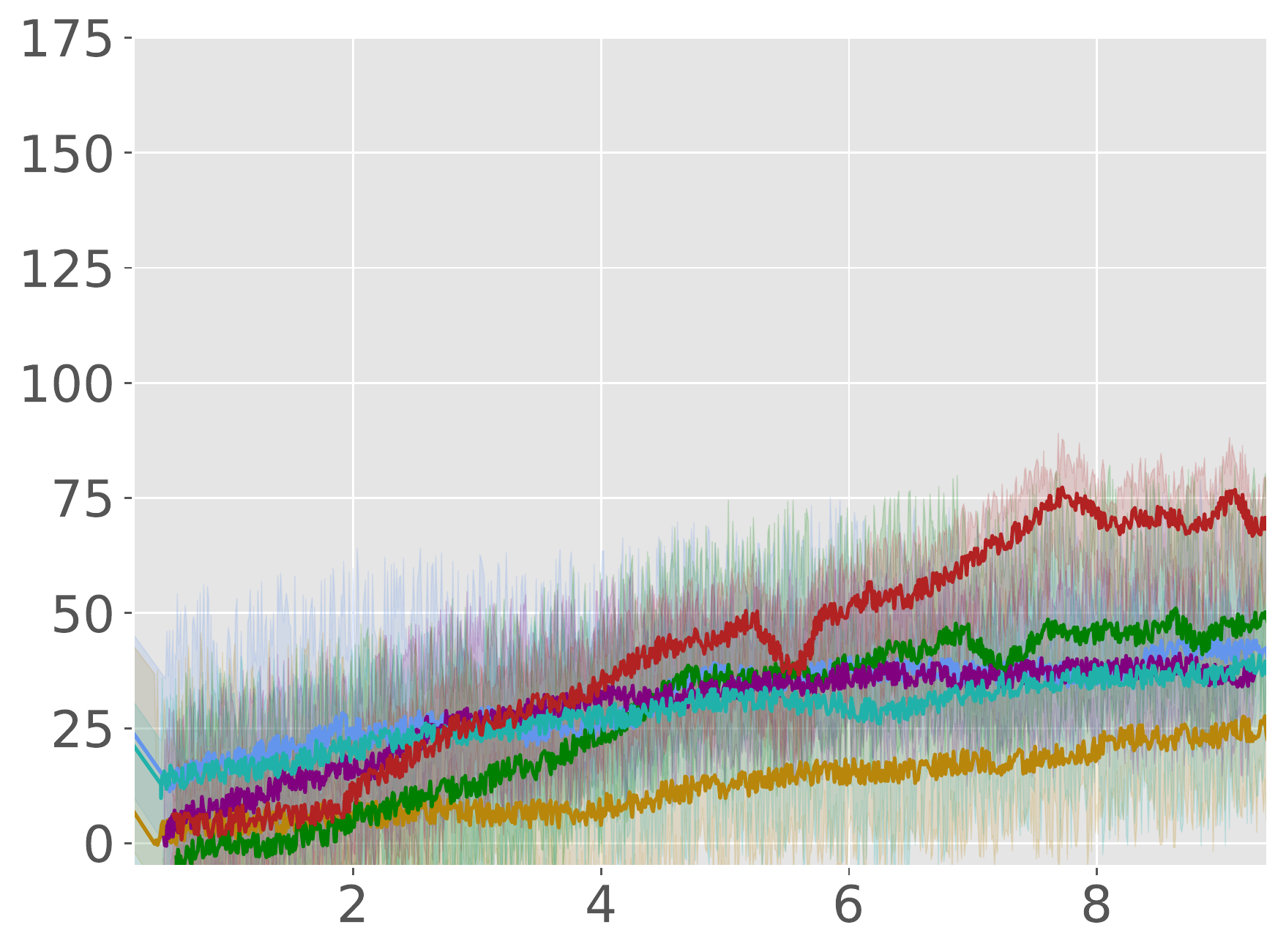}
        \label{fig:ant_gather_rewards}
    }
    {
        \includegraphics[width=0.32\textwidth]{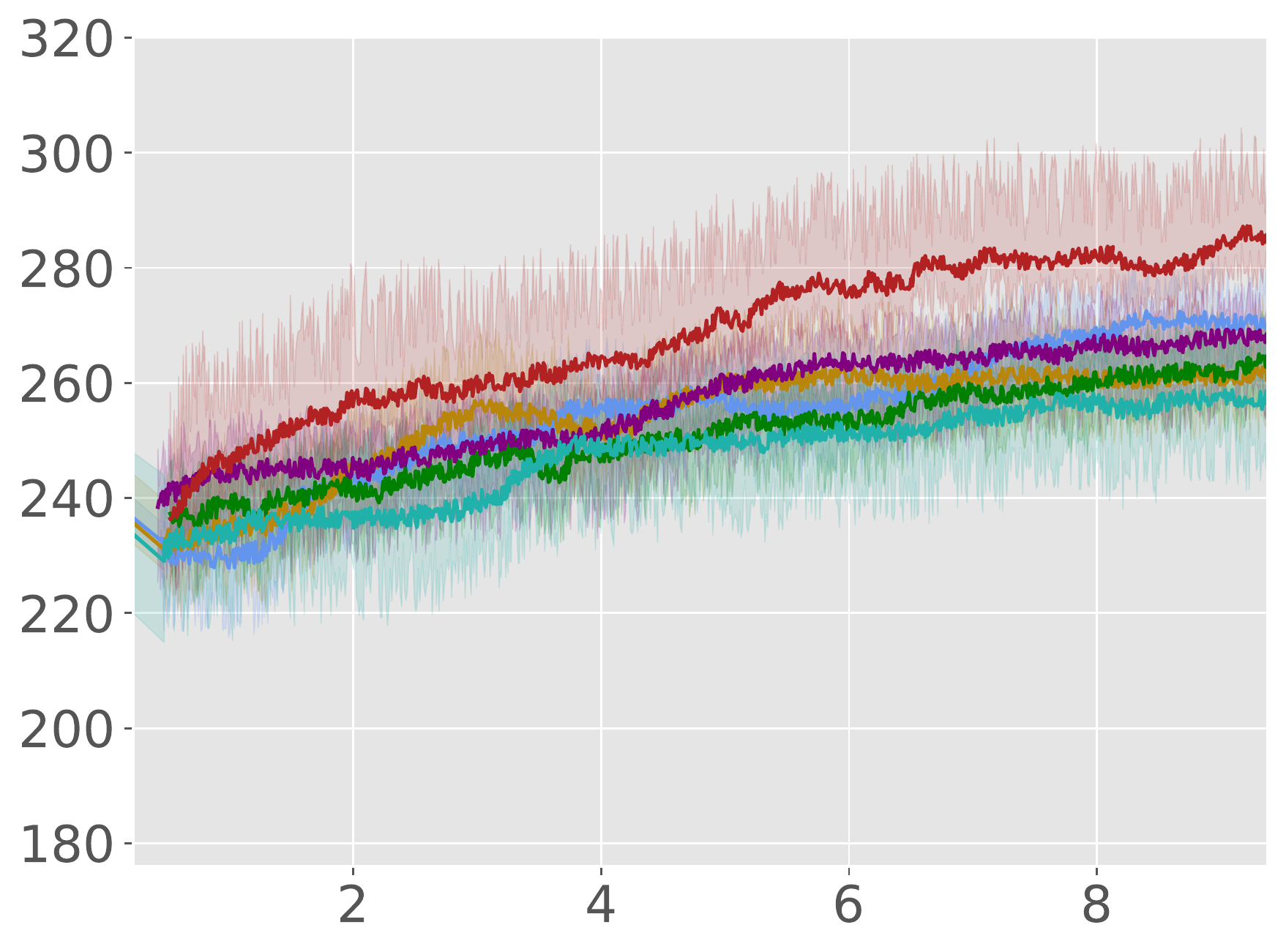}
        \label{fig:bottleneck_rewards}
    }
    {
        \includegraphics[width=0.32\textwidth]{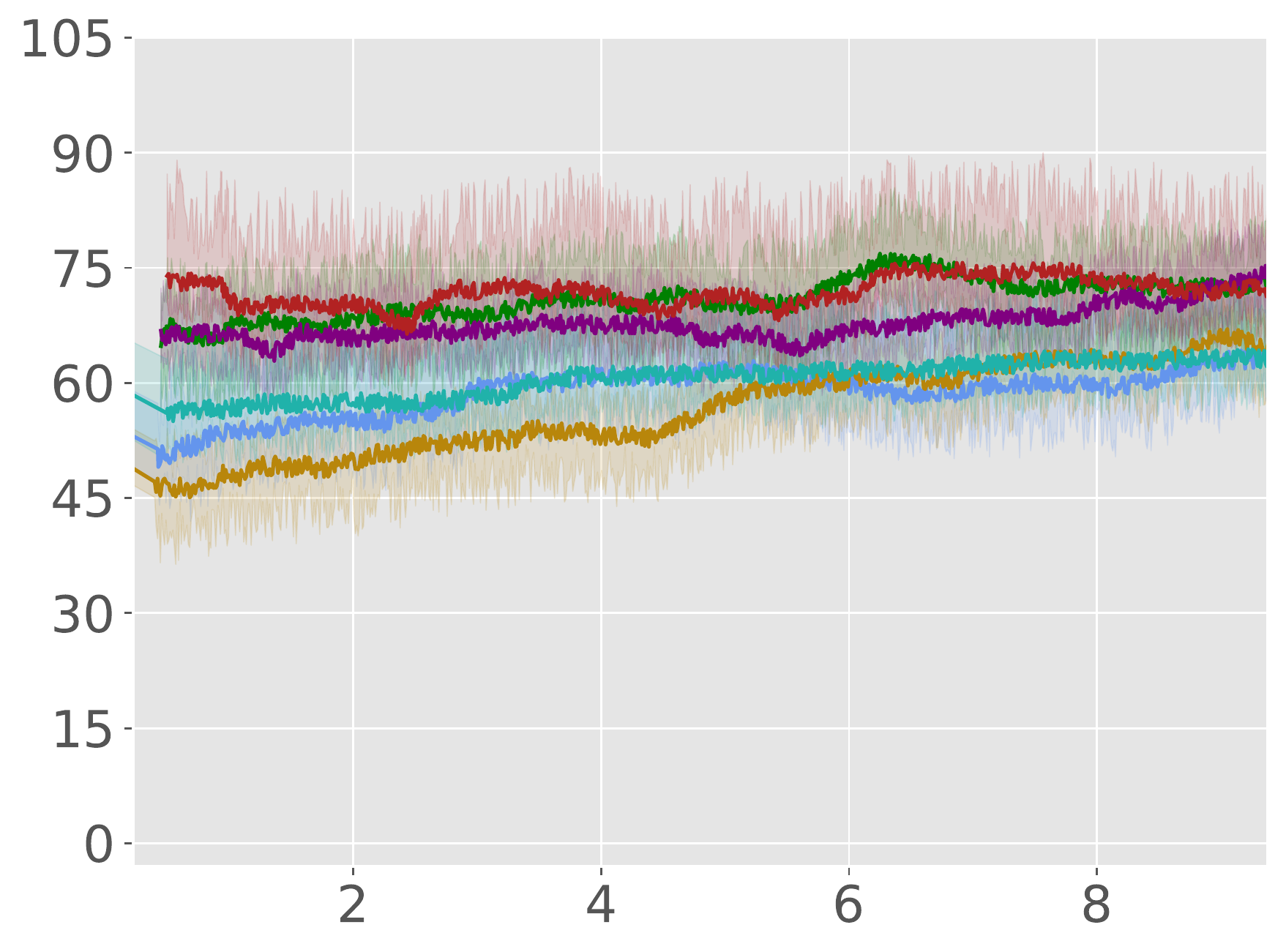}
        \label{fig:grid_rewards}
    }
    \rule{\linewidth}{0.5pt}
    \hspace{2cm} Average Constraint values: \hspace{1.75cm} \includegraphics[width=0.6\textwidth]{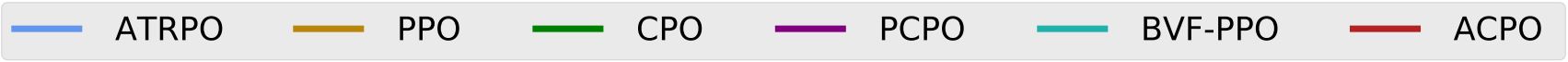} \newline
    % \subfloat[Point Circle]{
    %     \includegraphics[width=0.185\textwidth]{images/point_circle_costs.pdf}
    %     \label{fig:point_circle_costs}
    % }
    % \subfloat[Point Gather]{
    %     \includegraphics[width=0.185\textwidth]{images/point_gather_costs.pdf}
    %     \label{fig:point_gather_costs}
    % }
    % \subfloat[Ant Circle]{
    %     \includegraphics[width=0.23\textwidth]{images/ant_circle_costs.pdf}
    %     \label{fig:ant_circle_costs}
    % }
    \subfloat[Ant Gather]{
        \includegraphics[width=0.32\textwidth]{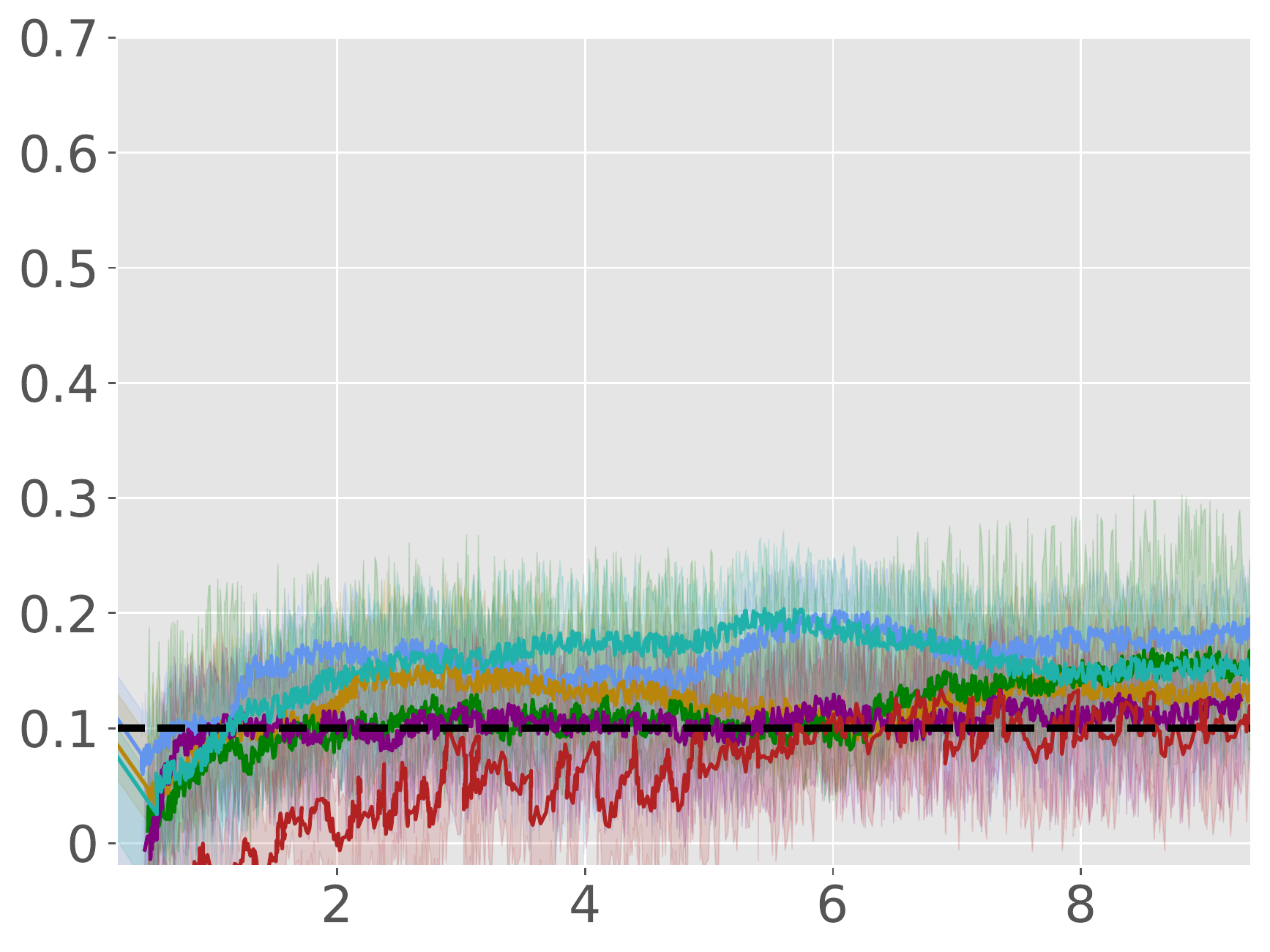}
        \label{fig:ant_gather_costs}
    }
    \subfloat[Bottleneck]{
        \includegraphics[width=0.32\textwidth]{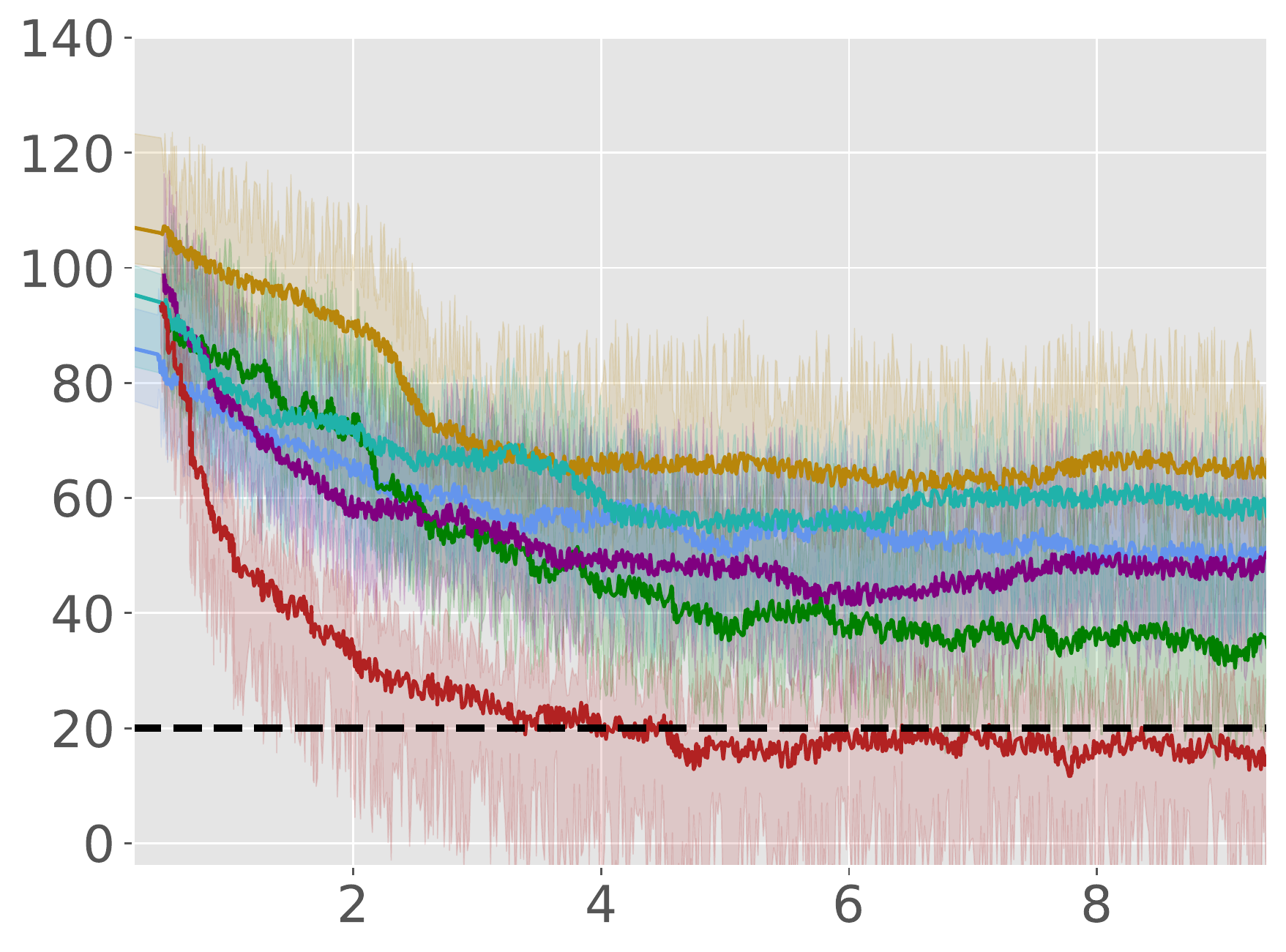}
        \label{fig:bottleneck_costs}
    }
    \subfloat[Grid]{
        \includegraphics[width=0.32\textwidth]{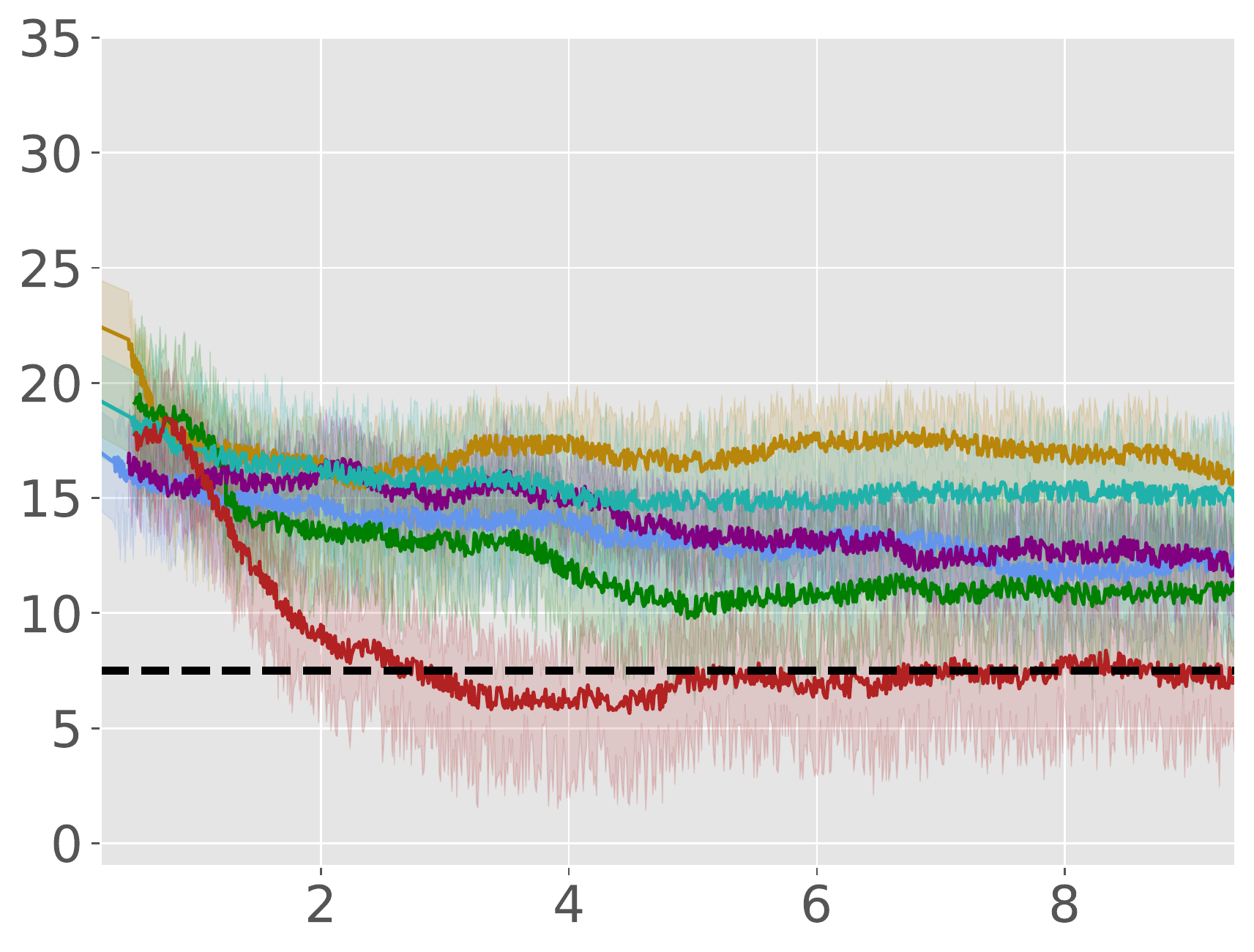}
        \label{fig:grid_costs}
    }
    \caption{The average reward and constraint cost function values vs iterations (in $10^{4}$) learning curves for some algorithm-task pairs. Solid lines in each figure are the empirical means, while the shaded area represents 1 standard deviation, all over 5 runs. The dashed line in constraint plots is the constraint threshold $l$. ATRPO and PPO are tested with constraints, which are included in their Lagrangian formulation. Additional results are available in Appendix \ref{appendix:additional_results}.}
    \label{fig:rewards_costs_comparison}
\end{figure*}

%%%%%%%%%%%%%%%%%%%%%%%%%%%%%%%%%%%%%%%%%%%%%%%%%%

\section{Empirical Results}
\label{sec:empirical}
We conducted a series of experiments to evaluate the relative performance of the ACPO algorithm and  answer the following questions: 
%\begin{itemize} \itemsep0em\item 
(i) Does ACPO learn a sequence of constraint satisfying policies while maximizing the average reward in the long run?
(ii) How does ACPO compare with the already existing constraint policy optimization algorithms which are applied with a large discount factor?
(iii) What are the factors that affect the performance of ACPO? 
%\end{itemize}

We work with the OpenAI Gym  environments to train the various learning agent on the following tasks - \textit{Gather}, \textit{Circle}, \textit{Grid}, and \textit{Bottleneck} tasks (see Figure \ref{fig:env_overview} in Appendix \ref{appendix:environments} for more details on the environments). For our experiments we only work with a single constraint with policy recovery using Eq.  \eqref{eq:acpo_recovery} (this is only a computational limitation; ACPO in principle can handle multiple constraints). We compare ACPO with the following baseline algorithms: CPO by \cite{achiam2017constrained}, ATRPO by \cite{zhang2021average}, PCPO by \cite{yang2020projection} (a close variant of CPO), BVF-PPO by \cite{pmlr-v119-satija20a} and PPO by \cite{schulman2017proximal}.

Although ATRPO and PPO originally do not incorporate constraints, for fair comparison, we introduce constraints using a Lagrangian. Also, CPO, PCPO and PPO are compared with $\gamma=0.999$. See Appendix \ref{appendix:experimental_details} for more details.

\subsection{Evaluation Details and Protocol}

For the Gather and Circle tasks we test two distinct agents:  a point-mass ($S \subseteq \Rbb^{9}, A \subseteq \Rbb^{2}$), and an ant robot ($S \subseteq \Rbb^{32}, A \subseteq \Rbb^{8}$). The agent in the Bottleneck task in $S \subseteq \Rbb^{71}, A \subseteq \Rbb^{16}$, and for the Grid task is $S \subseteq \Rbb^{96}, A \subseteq \Rbb^{4}$. We use two hidden layer neural networks to represent Gaussian policies for the tasks. For Gather and Circle, size is (64,32) for both layers, and for Grid and Bottleneck the layer sizes are (16,16) and (50,25). We set the step size $\delta$ to $10^{-4}$, and for each task, we conduct 5 runs to get the mean and standard deviation for reward objective and cost constraint values during training. We train CPO, PCPO, and PPO with the discounted objective, however, evaluation and comparison with BVF-PPO, ATRPO and ACPO\footnote{Code of the ACPO implementation will be made available on GitHub.} is done using the average reward objective (this is a standard evaluation scheme  as in \cite{schulman2015trust, wu2017scalable, vuong2019supervised}). 

For each environment, we train an agent for $10^{5}$ steps, and for every $10^{3}$ steps, we instantiate 10 evaluation trajectories with the current (deterministic) policy. For each of these trajectories, we calculate the trajectory average reward for the next $10^{3}$ steps and finally report the total average-reward as the mean of these 10 trajectories. Learning curves for the algorithms are compiled in Figure \ref{fig:rewards_costs_comparison} (for Point-Circle, Point-Gather, and Ant-Circle see Appendix \ref{appendix:additional_results}). 

Since there are two objectives  (rewards in the objective and costs in the constraints), we show the plots which maximize the reward objective while satisfying the cost constraint. See Appendix \ref{appendix:practical_acpo} and \ref{appendix:experimental_details} for more details.

% \noindent\textbf{Note.} Although it might seem interesting to train ATRPO and ACPO in the average-reward setting and evaluate them in the discounted setting, the motivation for such is lacking since there already exist state-of-the-art algorithms for the discounted setting, and the focus of this work is on the average reward setting.

\subsection{Performance Analysis}

From Figure \ref{fig:rewards_costs_comparison}, we can see that ACPO is able to improve the reward objective while having approximate constraint satisfaction on all tasks. In particular, ACPO is the only algorithm that best learns almost-constraint-satisfying maximum average-reward policies across all tasks: in a simple Gather environment, ACPO is able to almost exactly track the cost constraint values to within the given threshold $l$;  however, for the high dimensional Grid and Bottleneck environments we have more constraint violations due to complexity of the policy behavior. Regardless, in these environments, ACPO still outperforms all other baselines. 

\noindent\textbf{ACPO vs. CPO/PCPO.} For the Point-Gather environment (see Figure \ref{fig:appendix_rewards_costs_comparison}),  we see that initially ACPO and CPO/PCPO give relatively similar performance, but eventually ACPO improves over CPO and PCPO by 52.5\% and 36.1\% on average-rewards respectively. This superior performance does not come with more constraint violation. The Ant-Gather environment particularly brings out the effectiveness of ACPO where it shows 41.1\% and 61.5\% improvement over CPO and PCPO respectively, while satisfying the constraint. In the high dimensional Bottleneck and Grid environments, ACPO is particularly quick at optimizing for low constraint violations, while improving over PCPO and CPO in terms of average-reward. 
% This improvement may not seem as much, but when seen in the context of the average-reward improvement of ATRPO and PPO (21.1\% and 16.4\%) over PCPO, the effectiveness of ACPO for Bottleneck environment becomes apparent. 
% For the Grid environment, we observe that initially ACPO has high constraint violations, which is due to complexity of the environment and non-convexity of the constraint set. However, in the long run, ACPO stabilizes while giving non-decreasing average-reward, and overall the best performance of all the algorithms. 

\noindent\textbf{ACPO vs Lagrangian ATRPO/PPO.} One could suppose to use the state of the art unconstrained policy optimization algorithms with a Lagrangian formulation to solve the average-rewards CMDP problem in consideration, but we see that such an approach, although principled in theory, does not give satisfactory empirical results. This can be particularly seen in the Ant-Circle, Ant-Gather, Bottleneck, and Grid environments, where Lagrangian ATRPO and PPO give the least rewards, while not even satisfying the constraints. If ATRPO and PPO were used with constraints ignored, one would see higher rewards but even worse constraint violations, which are not useful.

\noindent\textbf{ACPO vs BVF-PPO.} BVF-PPO is a whole different formulation than the other baselines, as it translates the cumulative cost constraints into state-based constraints, which results in an almost-safe policy improvement method which maximizes returns at every step. However, we see that this approach fails to satisfy the constraints even in the moderately difficult Ant Gather environment, let alone the high dimensional Bottleneck and Grid environments.

% \noindent\textbf{PCPO vs. CPO.} Although the work of PCPO by \cite{yang2020projection} claims great that PCPO out performs CPO, we found that, while initially this is the case in all of the environments, but  in the long run, they achieve similar performance in average-rewards and average-constraint violations. This is  due to the fact that in both algorithms, the objective to be maximized is still the discounted advantage function, which leads to sub-par long-run performance.

% \noindent\textbf{ATRPO vs. PPO.} Comparing these unconstrained algorithms, we see that ATRPO does just as good as, if not better than PPO in all the tasks. Initially, they both give similar performance in terms of average-rewards but in the long-run, ATRPO performs better than PPO, since both are based on trust region methods but PPO prioritizes near future performance instead of the long-run average reward.

\subsection{Dependence of the Recovery Regime}
\label{subsec:recovery_regime}

% \begin{wrapfigure}{R}{0.52\textwidth}
\begin{figure}[t]
    \centering
    \subfloat[Rewards]{  \includegraphics[width=0.22\textwidth]{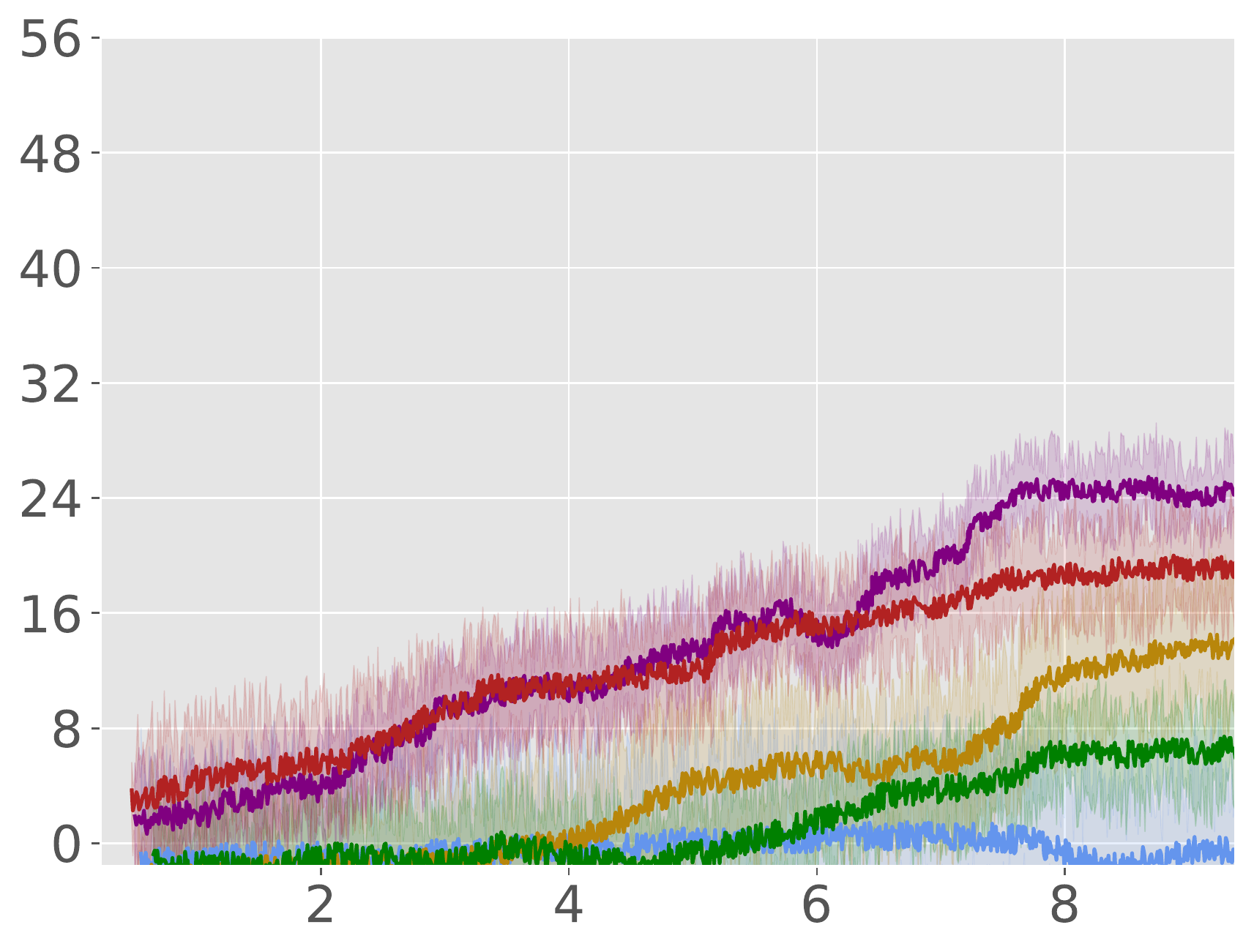}
        \label{fig:point_circle_hyper_rewards}
    }
    \subfloat[Costs]{
        \includegraphics[width=0.22\textwidth]{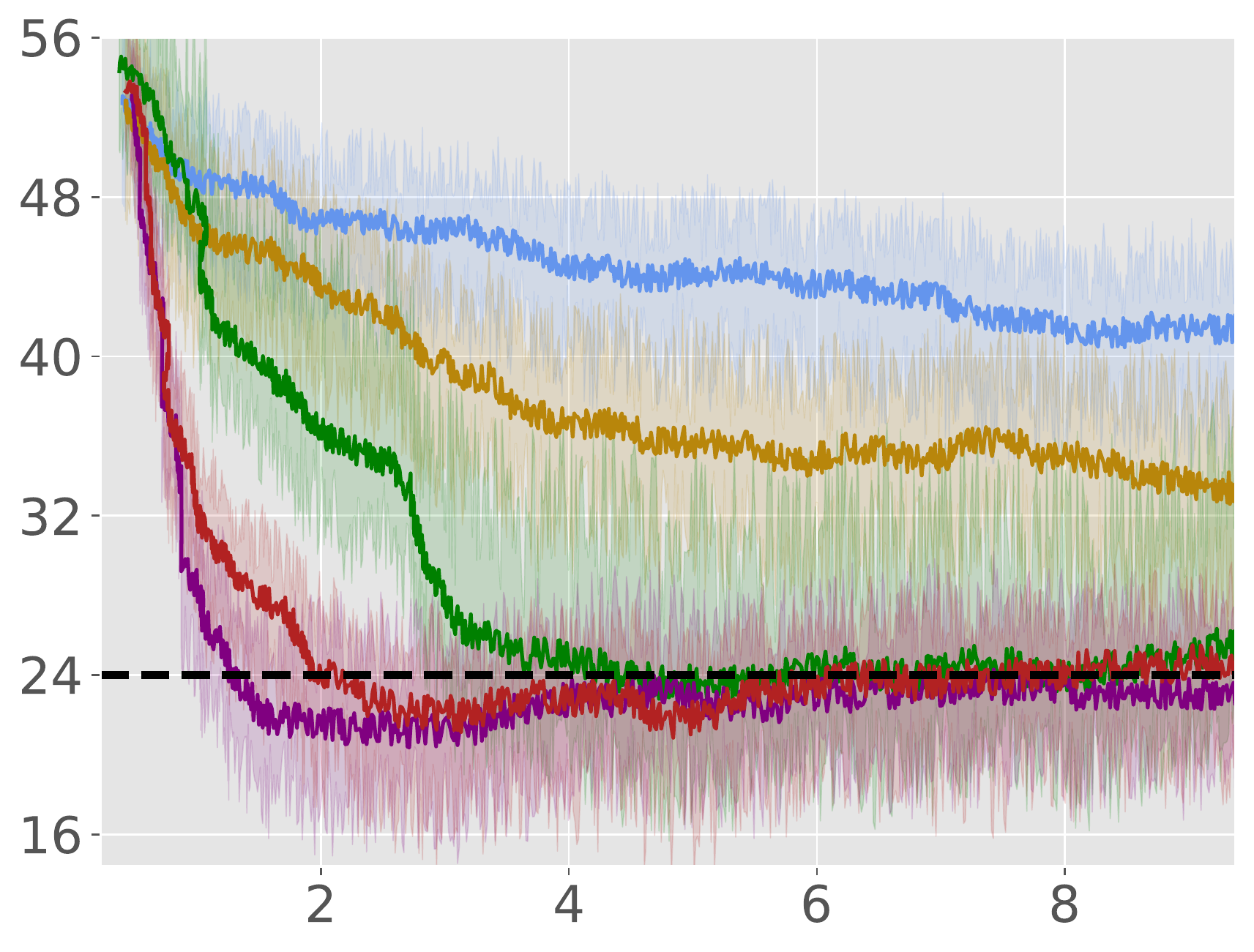}
        \label{fig:point_circle_hyper_costs}
    } 
    
     \includegraphics[width=0.45\textwidth]{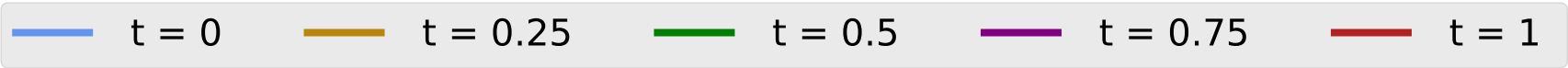}
    \caption[]{Comparison of performance of ACPO with different values of the hyperparameter $t$ in the Point-Circle environment. X-axis is iterations in $10^{4}$. See Appendix \ref{appendix:additional_results} for more details.}
    \label{fig:hyperparam_t_main}
\end{figure}
% \end{wrapfigure}

In Equation \eqref{eq:acpo_recovery} we introduced a hyperparameter $t$, which provides for an intuitive trade-off as follows: either we purely decrease the constraint violations ($t=1$), or we decrease the average-reward ($t=0$), which consequently decreases the constraint violation. The latter formulation is principled in that if we decrease rewards, we are bound to decrease constraints violation due to the nature of the environments. Figure \ref{fig:hyperparam_t_main} shows the experiments we conducted with varying $t$. With $t=1$, we obtain the same recovery scheme as that of \cite{achiam2017constrained}. Our results show that this scheme does not lead to the best performance, and that $t=0.75$ and $t=1$ perform the best across all tasks. See Appendix \ref{appendix:additional_results} for a detailed study.

%%%%%%%%%%%%%%%%%%%%%%%%%%%%%%%%%%%%%%%%%%%%%%%%%%

%%%%%%%%%%%%%%%%%%%%%%%%%%%%%%%%%%%%%%%%%%%%%%%%%%

\section{Conclusions}\label{sec:conclusions}

In this paper, we studied the problem of learning policies that maximize average-rewards for a given CMDP with average-cost constraints. We showed that the current algorithms with constraint violation bounds for the discounted setting do not generalize to the average setting. We then proposed a new algorithm, the Average-Constrained Policy Optimization (ACPO) that is inspired by the TRPO class of algorithms but based on theoretical sensitivity-type bounds for average-CMDPs we derive, and use in designing the algorithm. Our experimental results on a range of OpenAI Gym  environments (including some high dimensional ones) show the effectiveness of ACPO on ACMDP RL problems, as well as its superior empirical performance vis-a-vis some current alternatives. A direction for future work is implementation of ACPO to fully exploit the parallelization potential. 

\noindent\textbf{Impact:} This paper not only advances the field of RL theory and algorithms but also introduces a practical and scalable algorithm (supported by theory) that is of utility in many fields including LLMs, Diffusion Models, and robotic control. Currently, algorithms such as PPO are adapted for use simply because of lack of alternatives. 

\clearpage

\newpage

{\small
\bibliography{references}
\bibliographystyle{util/icml2024}
}

%%%%%%%%%%%%%%%%%%%%%%%%%%%%%%%%%%%%%%%%%%%%%%%%%%%%%%%%%%%%%%%%%%%%%%%%%%%%%%%
%%%%%%%%%%%%%%%%%%%%%%%%%%%%%%%%%%%%%%%%%%%%%%%%%%%%%%%%%%%%%%%%%%%%%%%%%%%%%%%
% APPENDIX
%%%%%%%%%%%%%%%%%%%%%%%%%%%%%%%%%%%%%%%%%%%%%%%%%%%%%%%%%%%%%%%%%%%%%%%%%%%%%%%
%%%%%%%%%%%%%%%%%%%%%%%%%%%%%%%%%%%%%%%%%%%%%%%%%%%%%%%%%%%%%%%%%%%%%%%%%%%%%%%
\newpage
\appendix
\onecolumn

\section{Appendix}
\label{sec:appendix}

\subsection{Proofs}

%%%%%%%%%%%%%%%%%%%%%%%%%%%%%%%%%%%%%%%

\begin{lemma}[Trivialization of Discounted Criterion Bounds]
\label{proof:cpo_bound_trivial}
Consider the policy performance bound of \cite{achiam2017constrained}, which says that for any two stationary policies $\pi$ and $\pi'$:
\begin{equation}\label
{eq:pol_imp_disc}
    J_{\gamma}(\pi')-J_{\gamma}(\pi) \geq \frac{1}{1-\gamma}\left[\E{\substack{s\sim \dpid \\ a\sim\pi'}}[\advd(s,a)] - \frac{2\gamma\epsilon^{\gamma}}{1-\gamma}\E{s\sim\dpid}\TV{\pi'}{\pi}[s]\right]
\end{equation}
where $\epsilon^{\gamma} = \max_s\left|\E{a\sim\pi'}[\advd(s,a)]\right|$. Then, the right hand side times $1-\gamma$ goes to negative infinity as $\gamma\to 1$.
\end{lemma}
\begin{proof}
Since $\dpid$ approaches the stationary distribution $\dpi$ as $\gamma\to 1$, we can multiply the right hand side of \eqref{eq:pol_imp_disc} by $(1-\gamma)$ and take the limit which gives us:
\begin{align*}
    &\lim_{\gamma\to 1}\left(\E{\substack{s\sim \dpid \\ a\sim\pi'}}[\advd(s,a)] \pm \frac{2\gamma\epsilon^{\gamma}}{1-\gamma}\E{s\sim\dpid}\TV{\pi'}{\pi}[s]\right) \\
    =& \E{\substack{s\sim \dpi \\ a\sim\pi'}}[\adv(s,a)] - 2\epsilon \E{s\sim\dpi}\big[\TV{\pi'}{\pi}[s]\big] \cdot \lim_{\gamma\to 1}\frac{\gamma}{1-\gamma} \\
    =& -\infty
\end{align*}
Here $\epsilon = \max_s\left|\E{a\sim\pi'}[\adv(s,a)]\right|$. The first equality is a standard result of $\lim_{\gamma\to 1}\advd(s,a) = \adv(s,a)$.
\end{proof}

%%%%%%%%%%%%%%%%%%%%%%%%%%%%%%%%%%%%%%%
\policydiff*
\begin{proof}
\label{proof:policy_diff_proof}
    We directly expand the right-hand side using the definition of the advantage function and  Bellman equation, which gives us:
\begin{align*}
    \E{\substack{s\sim d_{\pi'}\\ a\sim\pi'}}\left[\adv(s,a)\right] 
    &= \E{\substack{s\sim d_{\pi'}\\ a\sim\pi'}}\left[\qfunc(s,a) - \vfunc(s)\right] \\
    &= \E{\substack{s\sim d_{\pi'}\\ a\sim\pi' \\ s'\sim P(\cdot|s,a)}}\left[r(s,a,s') - J(\pi) + \vfunc(s') - \vfunc(s)\right] \\
    &= J(\pi') - J(\pi) + \underbrace{\E{\substack{s\sim d_{\pi'}\\ a\sim\pi'\\ s'\sim P(\cdot|s,a)}}[\vfunc(s')] - \E{s\sim d_{\pi'}}[\vfunc(s)]}_{A}
\end{align*}
Analyzing $A$, since $\dpip(s)$ is a stationary distribution:
\begin{align*}
    \E{\substack{s\sim d_{\pi'}\\ a\sim\pi'\\ s'\sim P(\cdot|s,a)}}[\vfunc(s')] &= \sum_{s}\dpip(s)\sum_{a}\pi'(a|s)\sum_{s'}P(s'|s,a)\vfunc(s') \\ &= \sum_{s}\dpip(s)\sum_{s'}P_{\pi'}(s'|s)\vfunc(s') = \sum_{s'}\dpip(s')\vfunc(s')
\end{align*}
Therefore, $A = \sum_{s'}\dpip(s')\vfunc(s') - \E{s\sim d_{\pi'}}[\vfunc(s)] = 0$. Hence, proved.
\end{proof}

%%%%%%%%%%%%%%%%%%%%%%%%%%%%%%%%%%%%%%%
\policyimpd*
\begin{proof}
\label{proof:policy_impd_proof}
\begin{align*}
    \left|J(\pi') - J(\pi) - \E{s\sim \dpi \\ a\sim\pi'}\left[\adv(s,a)\right]\right| &= \left|\E{s\sim d_{\pi'} \\ a\sim\pi'}\left[\adv(s,a)\right] - \E{s\sim \dpi\\ a\sim\pi'}\left[\adv(s,a)\right]\right| \tag{from Lemma \ref{lemma:policy_diff}} \\
    &=\left|\sum_s \E{a\sim\pi'}\left[\adv(s,a)\right]\left(\dpip(s) - \dpi(s)\right)\right| \\
    &\leq \sum_s\left|\E{a\sim\pi'}\left[\adv(s,a)\right]\left(\dpip(s) - \dpi(s)\right)\right| \\
    &\leq \max_s\left|\E{a\sim\pi'}\left[\adv(s,a)\right]\right|\norm{\dpip - \dpi}{1}{} \tag{Holder's inequality} \\
    &= 2\epsilon\TV{d_{\pi'}}{d_{\pi}}
\end{align*}
\end{proof}

%%%%%%%%%%%%%%%%%%%%%%%%%%%%%%%%%%%%%%%

\dandpi*
\begin{proof}
\label{proof:d_and_pi}
This proof takes ideas from Markov chain perturbation theory in \cite{cho2001comparison,hunter2005stationary, zhang2020average}. Firstly we state a standard result with $P_{\pi}^{\star} =\textbf{1}\dpi^T$
$$
 (d_{\pi'}-d_{\pi})^T(I-P_{\pi'}+P_{\pi'}^{\star}) = d_{\pi'}^T-d_{\pi}^T - d_{\pi'}^T + d_{\pi}^T P_{\pi'} 
    = d_{\pi}^T P_{\pi'} - d_{\pi}^T
    = d_{\pi}^T(P_{\pi'}-P_{\pi}).
$$

Denoting the fundamental matrix of the Markov chain $\Zpip =( I-P_{\pi'}+P_{\pi'}^{\star})^{-1}$ and the mean first passage time matrix $\Mpip = (I-\Zpip + E \Zpip_{\dg})D^{\pi'}$, and right multiplying the above by $(\Zpip)^{-1}$  we have, 
\begin{align}
\label{eq:d_and_pi_midstep}
d_{\pi'}^T-d_{\pi}^T = d_{\pi}^T(P_{\pi'}-P_{\pi})(I - \Mpip (D^{\pi'})^{-1}) &\implies d_{\pi'}-d_{\pi} = (I - \Mpip (D^{\pi'})^{-1})^T (P_{\pi'}^T-P_{\pi}^T)\dpi \\
\text{i.e.} \qquad \norm{d_{\pi'}-d_{\pi}}{1}{} &\leq \norm{(I - \Mpip (D^{\pi'})^{-1})^T (P_{\pi'}^T-P_{\pi}^T)\dpi}{1}{} \nonumber \tag{submultiplicative property} \\
\norm{d_{\pi'}-d_{\pi}}{1}{} &\leq \underbrace{\norm{(I - \Mpip (D^{\pi'})^{-1})}{\infty}{}}_{T_{1}} \underbrace{\norm{(P_{\pi'}^T-P_{\pi}^T)\dpi}{1}{}}_{T_{2}}    \tag{Holder's inequality}  \nonumber
\end{align}

We know that $\kappa^{\pi} = \text{Tr}(\Zpi)$ and from \cite{hunter2014mathematical}, we can write $T_{1}$ using the eigenvalues $\{\lambda_{\pi, i}\}_{i=1}^{|S|}$ of the underlying $P_{\pi}$ as
$$
T_{1} \leq \frac{1}{|S|} \sum_{i=2}^{|S|} \frac{1}{(1 - \lambda_{\pi, i})^{1/2}} \leq \max_i (1 - \lambda_{\pi, i})^{-1/2} = \sigma^{\pi} \leq \max_{\pi} \sigma^{\pi} = \sigma^{\star}.
$$

For $T_{2}$, we refer to the result by \cite{zhang2020average}, and provide the proof for completeness below.

\begin{align*}
    T_{2} &= \sum_{s'}\left|\sum_s\left(\sum_a P(s'|s,a)\pi'(a|s)-P(s'|s,a)\pi(a|s)\right)d_{\pi}(s)\right| \\
    &\leq \sum_{s',s}\left|\sum_a P(s'|s,a)(\pi'(a|s)-\pi(a|s))\right|d_{\pi}(s) \\
    &\leq \sum_{s,s',a}P(s'|s,a)\left|\pi'(a|s)-\pi(a|s)\right|d_{\pi}(s) \\
    &\leq \sum_{s,a}\left|\pi'(a|s)-\pi(a|s)\right|d_{\pi}(s) = 2 \E{s\sim d_{\pi}}[\TV{\pi'}{\pi}[s]]
\end{align*}
Combining these inequalities of $T_{1}$ and $T_{2}$, we get the desired result.
\end{proof}

\subsection{Performance and Constraint Bounds of Trust Region Approach}

Consider the trust region formulation in Equation \eqref{eq:acpo_trust}. To prove the policy performance bound when the current policy is infeasible (i.e., constraint-violating), we prove the KL divergence between $\pik$ and $\pikup$ for the KL divergence projection, along with other lemmas. We then prove our main theorem for the worst-case performance degradation.

\begin{lemma}
\label{lemma:feasible_kl_bound}
For a closed convex constraint set, if we have a constraint satisfying policy $\pik$ and the KL divergence $\E{s\sim d_{\pik}}\big[\KL{\pi_{k+1/2}}{\pik}[s]\big]$ of the `Improve' step is upper bounded by step size $\delta$, then after KL divergence projection of the `Project' step we have 
\[
\E{s\sim d_{\pik}}\big[\KL{\pikup}{\pik}[s]\big]\leq \delta.
\]
\end{lemma}
\begin{proof}

We make use of the fact that Bregman divergence (hence, KL divergence) projection onto the constraint set ($\in \Rbb^{d} \,, d \in \Nbb$) exists and is unique. Since $\pik$ is safe, we have $\pik$ in the constraint set, and $\pikup$ is the projection of $\pi_{k+\frac{1}{2}}$ onto the constraint set. Using the projection inequality, we have 
\begin{align*}
\E{s\sim d_{\pik}}\big[\KL{\pik}{\pikup}[s]\big] + \E{s\sim d_{\pik}}\big[\KL{\pikup}{\pi_{k+\frac{1}{2}}}[s]\big] \leq \E{s\sim d_{\pik}}\big[\KL{\pik}{\pi_{k+\frac{1}{2}}}[s]\big] & \\
\implies \E{s\sim d_{\pik}}\big[\KL{\pik}{\pikup}[s]\big] \leq \E{s\sim d_{\pik}}\big[\KL{\pik}{\pi_{k+\frac{1}{2}}}[s]\big] \leq \delta \tag{$\KL{\cdot}{\cdot} \geq 0$}.
\end{align*}
Since KL divergence is asymptotically symmetric when updating the policy within a local neighbourhood ($\delta << 1$), we have
$$
\E{s\sim d_{\pik}}\big[\KL{\pikup}{\pik}[s]\big] \leq 
\E{s\sim d_{\pik}}\big[\KL{\pi_{k+\frac{1}{2}}}{\pik}[s]\big] \leq \delta.
$$
\end{proof}

\begin{lemma}
\label{lemma:infeasible_kl_bound}

For a closed convex constraint set, if we have a constraint violating policy $\pik$ and the KL divergence $\E{s\sim d_{\pik}}\big[\KL{\pi_{k+1/2}}{\pik}[s]\big]$ of the first step is upper bounded by step size $\delta$, then after KL divergence projection of the second step we have 
\[
\E{s\sim d_{\pik}}\big[\KL{\pikup}{\pik}[s]\big]\leq \delta+V_{max},
\]
where $V_{max} = \max_{i} \alpha_{i}\beta_{i}^{2}, \,  \beta_{i} = [J_{C_{i}}(\pik)-l_{i}]_{+}$, $\alpha_{i} = \frac{1}{2a_{i}^T H^{-1}a_{i}},$ with
$a_{i}$ as the gradient of the cost advantage function corresponding to constraint $C_{i}$, and $H$ as the Hessian of the KL divergence constraint. \footnote{For any $x \in \Rbb, \, [x]_{+} := \max(0, x)$}.
\end{lemma}
\begin{proof}

Let the sublevel set of cost constraint function for the current infeasible policy $\pik$ be given as:
\[
L_{\pik}=\{\pi~|~J_{C_{i}}(\pi)+ \E{\substack{s\sim d_{\pik} \\ a\sim \pi}}[\wb{A}^{\pik}_{C_{i}}(s,a)] \leq J_{C_{i}}(\pik) \forAll i\}.
\]
This implies that the current policy $\pik$ lies in $L_{\pik}$. The constraint set onto which $\pi_{k+\frac{1}{2}}$ is projected onto is given by: $\{\pi~|~J_{C_{i}}(\pik)+ \E{\substack{s\sim d_{\pik} \\ a\sim \pi}}[\wb{A}^{\pik}_{C_{i}}(s,a)]\leq l_{i} \forAll i \}.$ Let $\pikup^L$ be the projection of $\pi_{k+\frac{1}{2}}$ onto $L_{\pik}.$

Note that the Bregman inequality of Lemma \ref{lemma:feasible_kl_bound} holds for any convex set in $\Rbb^{d} \,, d \in \Nbb$. This implies $\E{s\sim d_{\pik}}\big[\KL{\pikup^L}{\pik}[s]\big] \leq \delta$ since $\pik$ and $\pikup^L$ are both in $L_{\pik}$, which is also convex since the constraint functions are convex. Using the Three-point Lemma \footnote{For any $\phi$, the Bregman divergence identity: $D_\phi(x,y)+D_\phi(y,z)=D_\phi(x,z)+<\nabla \phi(z)-\nabla \phi(y),x-y>$}, for polices $\pik, \pikup,$ and $\pikup^L$, with $\varphi(\textbf{x}):=\sum_i x_i\log x_i$, we have

\begin{align}
\delta \geq  \E{s\sim d_{\pik}}\big[\KL{\pikup^L}{\pik})[s]\big]&=\E{s\sim d_{\pik}}\big[\KL{\pikup}{\pik}[s]\big] \nonumber\\ 
&-\E{s\sim d_{\pik}}\big[\KL{\pikup}{\pikup^L}[s]\big] \nonumber \\
&+\E{s\sim d_{\pik}}\big[(\nabla\varphi(\pik)-\nabla\varphi(\pikup^{L}))^T(\pikup-\pikup^L)[s]\big] \nonumber \\
\Rightarrow \E{s\sim d_{\pik}}\big[\KL{\pikup}{\pik}[s]\big]&\leq \delta + \underbrace{\E{s\sim d_{\pik}}\big[\KL{\pikup}{\pikup^L}[s]\big]}_{T_{1}} \nonumber \\
&- \underbrace{\E{s\sim d_{\pik}}\big[(\nabla\varphi(\pik)-\nabla\varphi(\pikup^L))^T(\pikup-\pikup^L)[s]\big]}_{T_{2}}. 
\end{align}

If the constraint violations of the current policy $\pik$ are small, i.e., $J_{C_{i}}(\pik)-l_{i} = b_{i} $ is small for all $i$, then $T_{1}$ can be approximated by a second order expansion. We analyze $T_{1}$ for any constraint $C_{i}$ and then bound it over all the constraints. As before we overload the notation with $\pik = \pi_{\theta_{k}} = \theta_{k}$ to write. For any constraint $C_{i}$, we can write $T^{i}_{1}$ as the expected KL divergence if projection was onto the constraint set of $C_{i}$ i.e.

\begin{align*}
    T^{i}_{1} \approx \frac{1}{2}(\pikup-\pikup^L)^{T} H(\pikup-\pikup^L) & =\frac{1}{2} \Big(\frac{\beta_{i}}{a_{i}^T H^{-1}a_{i}} H^{-1}a_{i}\Big)^T H\Big(\frac{\beta_{i}}{a_{i}^T H^{-1}a_{i}}H^{-1}a_{i}\Big) \\ &= \frac{\beta_{i}^2}{2a_{i}^T H^{-1}a_{i}} = \alpha_{i} \beta_{i}^2,
\end{align*}

where the second equality is a result of the trust region guarantee (see \cite{schulman2015trust} for more details). Finally we invoke the projection result from \cite{achiam2017advanced} which uses Dykstra's Alternating Projection algorithm from \cite{tibshirani2017dykstra} to bound this projection, i.e., $T_{1} \leq \max_{i} T^{i}_{1} \approx \max_{i} \alpha_{i} \beta_{i}^2$.

And since $\delta$ is small, we have $\nabla\varphi(\pik)-\nabla\varphi(\pikup^{L})\approx 0$ given $s$. Thus, $T_{2} \approx 0$. Combining all of the above, we have $\E{s\sim d_{\pik}}\big[\KL{\pikup}{\pik}[s]\big]\leq \delta+V_{max}.$
\end{proof}

\trustdegradationviolation*
\begin{proof}
\label{proof:trust_proof}

Since $\avKL{\polk}{\polk}=0$, $\polk$ is feasible. The objective value is 0 for $\pol=\polk$. The bound follows from Equation \eqref{eq:avg_policy_imp} and Equation \eqref{eq:tv-kl} where the average KL i.e. $\E{s\sim d_{\pik}}\big[\KL{\pikup}{\pik}[s]\big]$ is bounded by $\delta+V_{max}$ from Lemma \ref{lemma:infeasible_kl_bound}.

Similar to Corollary \ref{cor:thm:avg_constraint_imp}, expressions for the auxiliary cost constraints also follow immediately as the second result.

\remark{Remark} If we look at proof as given by \cite{zhang2020average} in Section 5 of their paper, with the distinction now that $\delta$ is replaced by $\delta + V_{max}$, we have the same result. Our worse bound is due to the constrained nature of our setting, which is intuitive in the sense that for the sake of satisfying constraints, we undergo a worse worst-case performance degradation.

\end{proof}

\subsection{Approximate ACPO}
\label{appendix:approx_acpo}

\subsubsection{Policy Recovery Routine}
As described in Section \ref{sec:recovery}, we need a recovery routine in case the updated policy $\pi_{k+1/2}$ is not approximate constraint satisfying. In this case, the optimization problem is inspired from a simple trust region approach by \cite{schulman2015trust}. Since we only deal with one constraint in the practical implementation of ACPO, the recovery rule is obtained by solving the following problem:

$$
\begin{array}{ll}
\min _{\theta} & c+a^{T}\left(\theta-\theta_{k}\right) \\
\text { s.t. } & \frac{1}{2}\left(\theta-\theta_{k}\right)^{T} H\left(\theta-\theta_{k}\right) \leq \delta .
\end{array}
$$

Let $x=\theta-\theta_{k}$, then the dual function $L(x, \lambda)$ is given by: $L(x, \lambda)=c+a^{T} x+\lambda\left(\frac{1}{2} x^{T} H x-\delta\right)$. Now,

$$
\frac{L}{\partial x}=a+\lambda(H x)=0 \Longrightarrow x=-\frac{1}{\lambda} H^{-1} a .
$$

$x$ obtained above should satisfy the trust-region constraint:

$$
\begin{aligned}
\frac{1}{2}\left(-\frac{1}{\lambda} H^{-1} a\right)^{T} H\left(-\frac{1}{\lambda} H^{-1} a\right) & \leq \delta \\
\Longrightarrow \quad \frac{1}{2} \cdot \frac{1}{\lambda^{2}} \cdot a^{T} H^{-1} a & \leq \delta \\
\Longrightarrow \sqrt{\frac{a^{T} H^{-1} a}{2 \delta}} & \leq \lambda .
\end{aligned}
$$

Therefore, the update rule in case of infeasibility takes the form $ \theta=\theta_{k}-\sqrt{\frac{2 \delta}{a^{T} H^{-1} a}} H^{-1} a$. We augment this rule with the gradient of the reward advantage function as well, so the final recovery is 
$$
\theta_{k+1/2} = \theta_{k} - \sqrt{2\delta} \bigg[t \cdot \frac{H^{-1} a}{\sqrt{a^T H^{-1} a}} + (1-t) \cdot \frac{H^{-1} g}{\sqrt{g^T H^{-1} g}} \bigg] \quad ; \quad t \in [0,1]
$$

\subsubsection{Line Search}
Because of approximation error, the proposed update may not satisfy the constraints in Eq. \eqref{eq:acpo_trust}. Constraint satisfaction is enforced via line search, so the final update is $$ \theta_{k+1} = \theta_k + s^j \left(\theta_{k+1/2} - \theta_k\right),$$ where $s \in (0,1)$ is the backtracking coefficient and $j \in \{0, ..., L\}$ is the smallest integer for which $\pi_{k+1}$ satisfies the constraints in Equation \ref{eq:acpo_trust}. Here, $L$ is a finite backtracking budget; if no proposed policy satisfies the constraints after $L$ backtracking steps, no update occurs.

\subsection{Practical ACPO}
\label{appendix:practical_acpo}

As explained in Section \ref{sec:acpo_implementation}, we use the below problem formulation, which uses first-order Taylor approximation on the objective and second-order approximation on the KL constraint \footnote{The gradient and first-order Taylor approximation of $\avKL{\pol}{\polk}$ at $\theta=\theta_k$ is zero.} around $\theta_k$, given small $\delta$: 
\begin{equation}
\begin{aligned}
\max_{\theta} \;\;\;&  g^T (\theta - \theta_k)  \\
\text{s.t. } \;\;\;& c_i + a_i^T (\theta - \theta_k) \leq 0, \forAll i \qquad ; \qquad  \tfrac{1}{2} (\theta - \theta_k)^T H (\theta - \theta_k) \leq \delta.
\end{aligned}
\end{equation}
where
\begin{align*}
    g &:= \E{\substack{s\sim\dpolk \\ a\sim\polk}}\left[\grad\log\pol(a|s)|_{\theta=\theta_k}\wb{A}^{\polk}(s,a)\right] \qquad &; \qquad  c_i &:= J_{C_i}(\theta_k) - l_i  \forAll i \\
    a_i &:= \E{\substack{s\sim\dpolk \\ a\sim\polk}}\left[\grad\log\pol(a|s)|_{\theta=\theta_k}\wb{A}^{\polk}_{C_{i}}(s,a)\right] \qquad &; \qquad  H &:= \E{\substack{s\sim\dpolk \\ a\sim\polk}}\left[\grad\log\pol(a|s)|_{\theta=\theta_k}\grad\log\pol(a|s)|_{\theta=\theta_k}^T\right] 
\end{align*}

Similar to the work of \cite{achiam2017constrained}, $g$, $a_{i}$, and $H$ can be approximated using samples drawn from the policy $\polk$. The Hessian $H$ is identical to the Hessian $H$ used by \cite{achiam2017constrained} and \cite{zhang2020average}. However, the definitons $g$ and $a_{i}$'s are different since they include the average reward advantage functions, $\wb{A}^{\polk}(s,a) =   \wb{Q}^{\polk}(s,a) -  \wb{V}^{\polk}(s)$.

Since rewards and cost advantage functions can be approximated independently, we use the framework of \cite{zhang2020average} to do so. We describe the process of estimation of rewards advantage function, and the same procedure can be used for the cost advantage functions as well. Specifically, first approximate the average-reward bias $\wb{V}^{\polk}(s)$ and then use a one-step TD backup to estimate the action-bias function. Concretely, using the average reward Bellman equation gives
\begin{equation}\label{eq:bellman_objective}
    \wb{A}^{\polk}(s,a) =  r(s,a) - J(\polk) + \E{s'\sim P(\cdot|s,a)}\left[\wb{V}^{\polk}(s')\right] - \wb{V}^{\polk}(s)
\end{equation}

This expression involves the average-reward bias $\wb{V}^{\polk}(s)$, which we can approximated using the standard critic network $\wb{V}_{\phi_k}(s)$. However, in practice we use the average-reward version of the Generalized Advantage Estimator (GAE) from \cite{schulman2016high}, similar to \cite{zhang2020average}. Hence, we refer the reader to that paper for detailed explanation, but provide an overview below for completeness.

Let $ \hat{J}_{\pi} = \frac{1}{N}\sum_{t=1}^N r_{t} $ denote the estimated average reward. The Monte Carlo target for the average reward value function is $\vtarg_t = \sum_{t'=t}^N (r_{t} - \hat{J}_{\pi})$ and the bootstrapped target is $\vtarg_t = r_{t} - \hat{J}_{\pi} +\vphi(s_{t+1})$.

The Monte Carlo and Bootstrap estimators for the average reward advantage function are:
\begin{align*}
    \hat{A}_{\text{MC}}^{\pi}(s_{t}, a_{t}) = \sum_{t'=t}^N (r_{t} - \hat{J}_{\pi}) -  \vphi(s_{t}) \qquad ; \qquad 
    \hat{A}_{\text{BS}}^{\pi}(s_{t}, a_{t}) = r_{i,t}-\hat{J}_{\pi} + \vphi(s_{t+1}) - \vphi(s_{t})
\end{align*}
We can similarly extend the GAE to the average reward setting:
\begin{equation}\label{eq:avg_GAE}
    \hat{A}_{\text{GAE}}(s_{t},a_{t}) = \sum_{t'=t}^N \lambda^{t'-t}\delta_{t'} \qquad , \qquad \delta_{t'} = r_{t'} - \hat{J}_{\pi}+ \vphi(s_{t'+1}) - \vphi(s_{t'}).
\end{equation}

and set the target for the value function to $ \vtarg_t = r_t - \hat{J}_{\pi}+ \vphi(s_{t+1}) + \sum_{t'=t+1}^N \lambda^{t'-t}\delta_{t'}$.

\subsection{Experimental Details}
\label{appendix:experimental_details}

For detailed explanation of Point-Circle, Point-Gather, Ant-Circle, and Ant-Gather tasks, please refer to \cite{achiam2017constrained}. For detailed explanation of Bottleneck and Grid tasks, please refer to \cite{vinitsky2018benchmarks}. For the simulations in the Gather and Circle tasks, we use neural network baselines with the same architecture and activation functions as the policy networks. For the simulations in the Grid and Bottleneck tasks, we use linear baselines. For all experiments we use Gaussian neural policies whose outputs are the mean vectors and variances are separate parameters to be learned. Seeds used for generating evaluation trajectories are different from those used for training.

For comparison of different algorithms, we make use of CPO, PCPO, ATRPO, and PPO implementations taken from \url{https://github.com/rll/rllab} and \url{https://github.com/openai/safety-starter-agents}. Even the hyperparameters are selected so as to showcase the best performance of other algorithms for fair comparison. The choice of the hyperparameters given below is inspired by the original papers since we wanted to understand the performance of the average reward case. 

We use settings which are common in all open-source implementations of the algorithms, such as normalizing the states by the running mean and standard deviation before being fed into the neural network and similarly normalizing the advantage values (for both rewards and constraints) by their batch means and standard deviations before being used for policy updates. Table \ref{tab:hyperparameters} summarizes the hyperparameters below.

\begin{table}[ht]
    \centering
    \caption{Hyperparameter Setup}
    \vskip 0.15in
    \begin{tabular}{l l l}
    \toprule
       Hyperparameter & PPO/ATRPO & CPO/PCPO/ACPO  \\
        \midrule
    No. of hidden layers & 2 & 2  \\
      Activation & $\tanh$ & $\tanh$ \\
      Initial log std & -0.5 & -1 \\
      Batch size & 2500 & 2500 \\
      GAE parameter (reward) & 0.95 & 0.95  \\ 
      GAE parameter (cost) & N/A & 0.95  \\
      Trust region step size $\delta$ & $10^{-4}$ & $10^{-4}$ \\
      Learning rate for policy & $2\times 10^{-4}$ & $2\times 10^{-4}$ \\
      Learning rate for reward critic net & $2\times 10^{-4}$ & $2\times 10^{-4}$  \\
      Learning rate for cost critic net & N/A & $2\times 10^{-4}$  \\
       Backtracking coeff.  & 0.75 & 0.75  \\
       Max backtracking iterations & 10 & 10 \\
       Max conjugate gradient iterations & 10 & 10 \\
       Recovery regime parameter $t$ & N/A & 0.75 \\
       \bottomrule
    \end{tabular}
    \label{tab:hyperparameters}
\end{table}

For the Lagrangian formulation of ATRPO and PPO, note that the original papers do not provide any blueprint for formulating the Lagrangian, and even CPO and PCPO use \emph{unconstrained} TRPO for benchmarking. However, we feel that this is unfair to these algorithms as they can possibly perform better with a Lagrangian formulation in an average-reward CMDP setting. To this extent, we introduced a Lagrangian parameter $\ell \in [0,1]$ that balances the rewards and constraints in the final objective function. More specifically, Equation \eqref{eq:acpo_approx} for a single constraint now becomes

\begin{equation}
\label{eq:lagrangian_balanced}
\begin{aligned}
\max_{\theta} \;\;\;&  (1-\ell) g^T (\theta - \theta_k) - \ell \big[ \big(c_1 + a_1^T (\theta - \theta_k)\big) + \big( \tfrac{1}{2} (\theta - \theta_k)^T H (\theta - \theta_k) - \delta \big)   \big].
\end{aligned}
\end{equation}

\paragraph{Note.} The authors of the ATRPO and PPO do not suggest any principled approach for finding an optimal $\ell$. Hence, the choice of the Lagrangian parameter $\ell$ is completely empirical and is selected such that these algorithms achieve maximum rewards while satisfying the constraints. Also see in Figure \ref{fig:rewards_costs_comparison}, for Ant-Gather, Bottleneck, and Grid environments, where the constraints cannot be satisfied for \emph{any} value of $\ell$, we include the results for a specific value of $\ell$ for illustrative purposes, as detailed in Table \ref{tab:lagrangian_parameters}.

\begin{table}[ht]
    \centering
    \caption{Lagrangian parameter $\ell$ for ATRPO and PPO}
    \vskip 0.15in
    \begin{tabular}{l c c c c c}
    \toprule
       Algorithm & Point-Gather & Ant-Circle & Ant-Gather & Bottleneck & Grid  \\
        \midrule
    ATRPO & 0.50 & 0.60 & 0.45 & 0.50 & 0.45 \\
      PPO & 0.55 & 0.50 & 0.50 & 0.50 & 0.60 \\
       \bottomrule
    \end{tabular}
    \label{tab:lagrangian_parameters}
\end{table}

\subsection{Experimental Addendum}
\label{appendix:additional_results}

\subsubsection{Environments}
\label{appendix:environments}

All environments tested on are illustrated in Figure \ref{fig:env_overview}, along with a detailed description of each.

\begin{figure}[h]
    \centering
    \subfloat[Circle]{
        \includegraphics[width=0.2\textwidth]{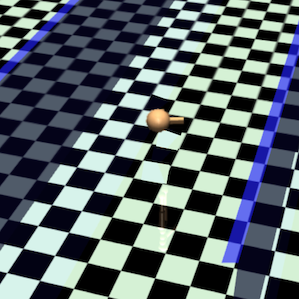}
        \label{fig:circle}
    }
    \subfloat[Gather]{
        \includegraphics[width=0.2\textwidth]{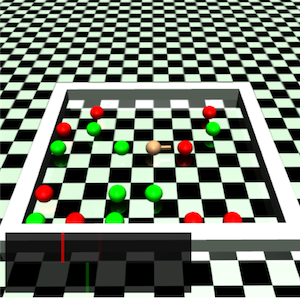}
        \label{fig:gather}
    }
    \subfloat[Grid]{
        \includegraphics[width=0.2\textwidth]{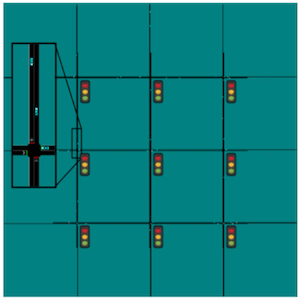}
        \label{fig:grid}
    }
    \subfloat[Bottleneck]{
        \includegraphics[width=0.2\textwidth, angle =90]{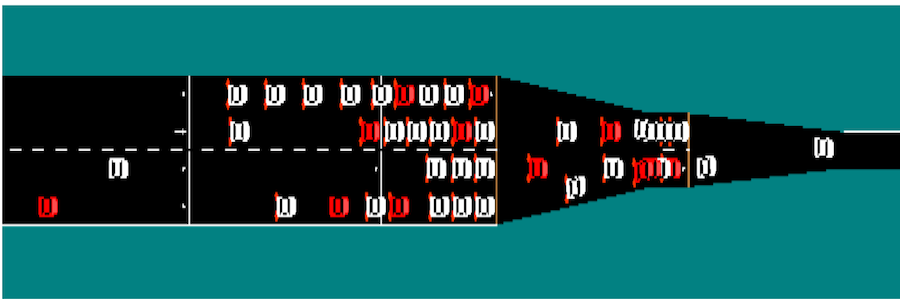}
        \label{fig:bottleneck}
    }
    \caption[]{The Circle, Gather, Grid, and Bottleneck tasks. (a) Circle: The agent is rewarded for moving in a specified circle but is penalized if the diameter of the circle is larger than some value as in \cite{achiam2017constrained}.  (b) Gather: The agent is rewarded for collecting the green balls while penalized to gather red balls as in \cite{achiam2017constrained}. (c) Grid: The agent controls traffic lights in a 3x3 road network and is rewarded for high traffic throughput but is constrained to let lights be red for at most 5 consecutive seconds as in \cite{vinitsky2018benchmarks}. (d) Botteneck: The agent controls vehicles (red) in a merging traffic situation and is rewarded for maximizing the number of vehicles that pass through but is constrained to ensure that white vehicles (not controlled by agent) have ``low'' speed for no more than 10 seconds as in \cite{vinitsky2018benchmarks}.}
    \label{fig:env_overview}
\end{figure}

\subsubsection{Learning Curves}

Due to space restrictions, we present the learning curves for the remaining environments in Figure \ref{fig:appendix_rewards_costs_comparison}.

\begin{figure*}[ht!]
\centering
    \hspace{-6cm} Average Rewards: \newline

    % {
    %     \includegraphics[width=0.18\textwidth]{images/point_circle_rewards.pdf}
    %     \label{fig:point_circle_rewards}
    % }
    {
        \includegraphics[width=0.3\textwidth]{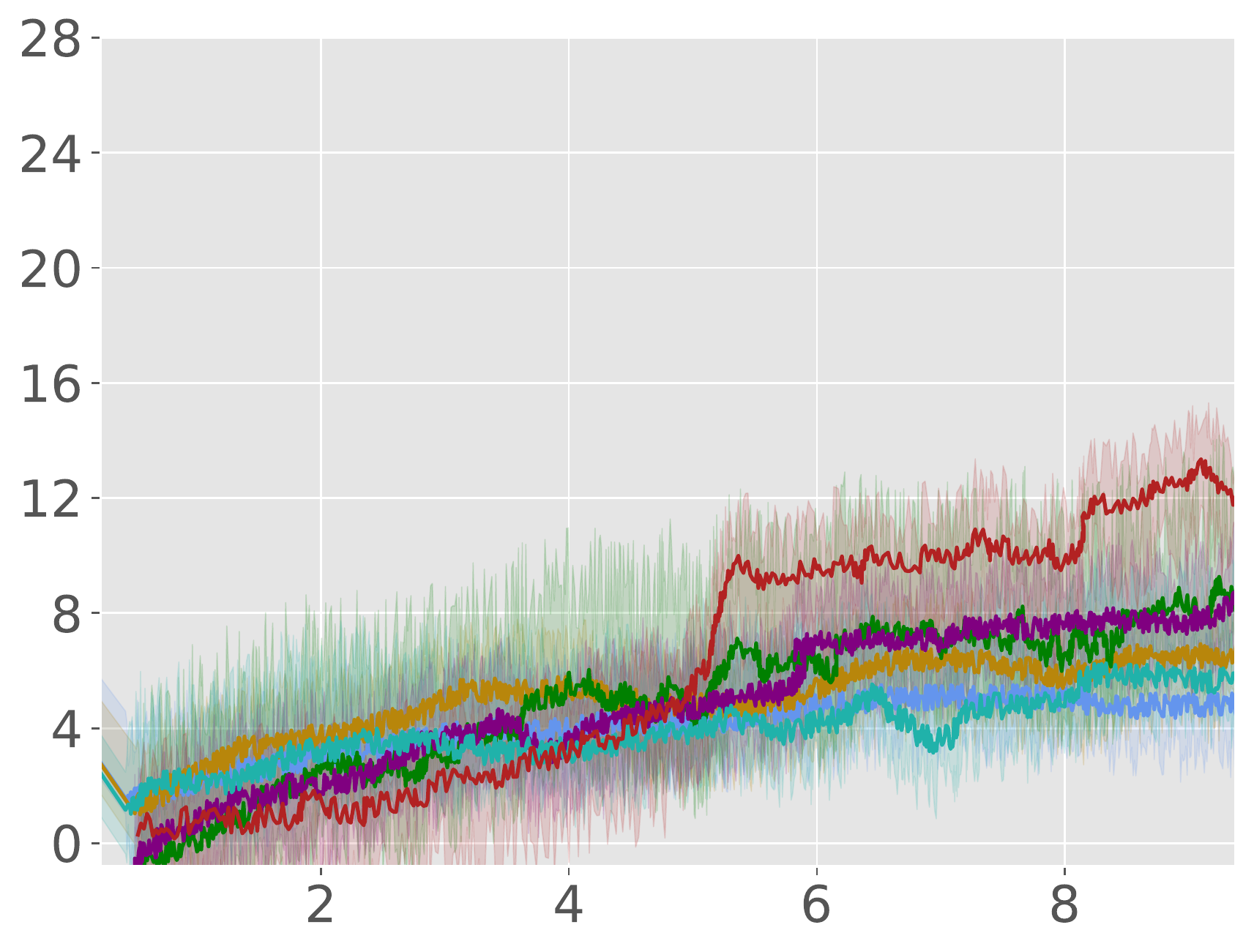}
        \label{fig:point_gather_rewards}
    }
    {
        \includegraphics[width=0.3\textwidth]{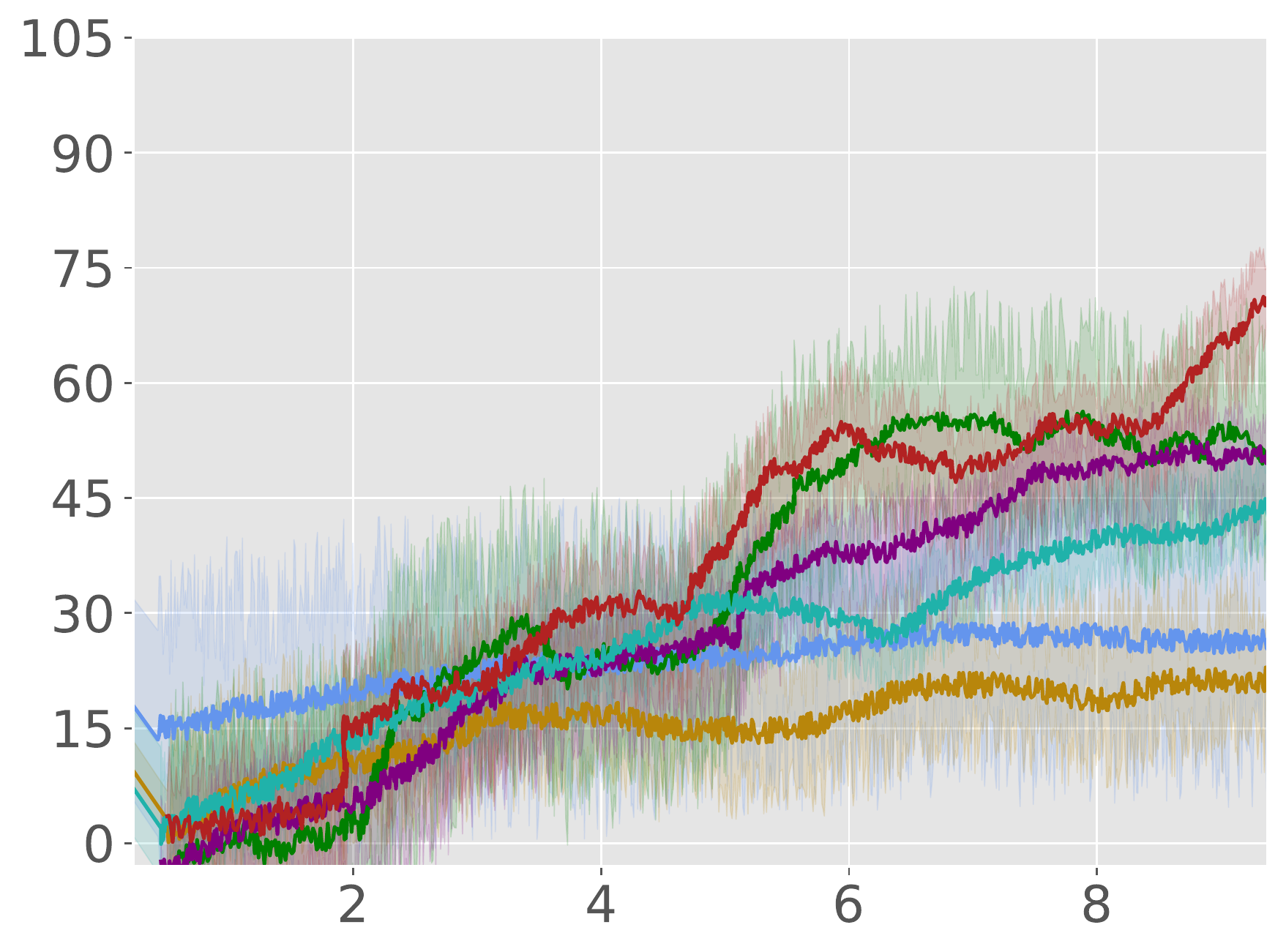}
        \label{fig:ant_circle_rewards}
    }

    \rule{\linewidth}{0.5pt}
    \hspace{2cm} Average Constraint values: \hspace{2.5cm} \includegraphics[width=0.55\textwidth]{images/rewards_costs_legend.JPG} \newline
    % \subfloat[Point Circle]{
    %     \includegraphics[width=0.185\textwidth]{images/point_circle_costs.pdf}
    %     \label{fig:point_circle_costs}
    % }
    \subfloat[Point Gather]{
        \includegraphics[width=0.3\textwidth]{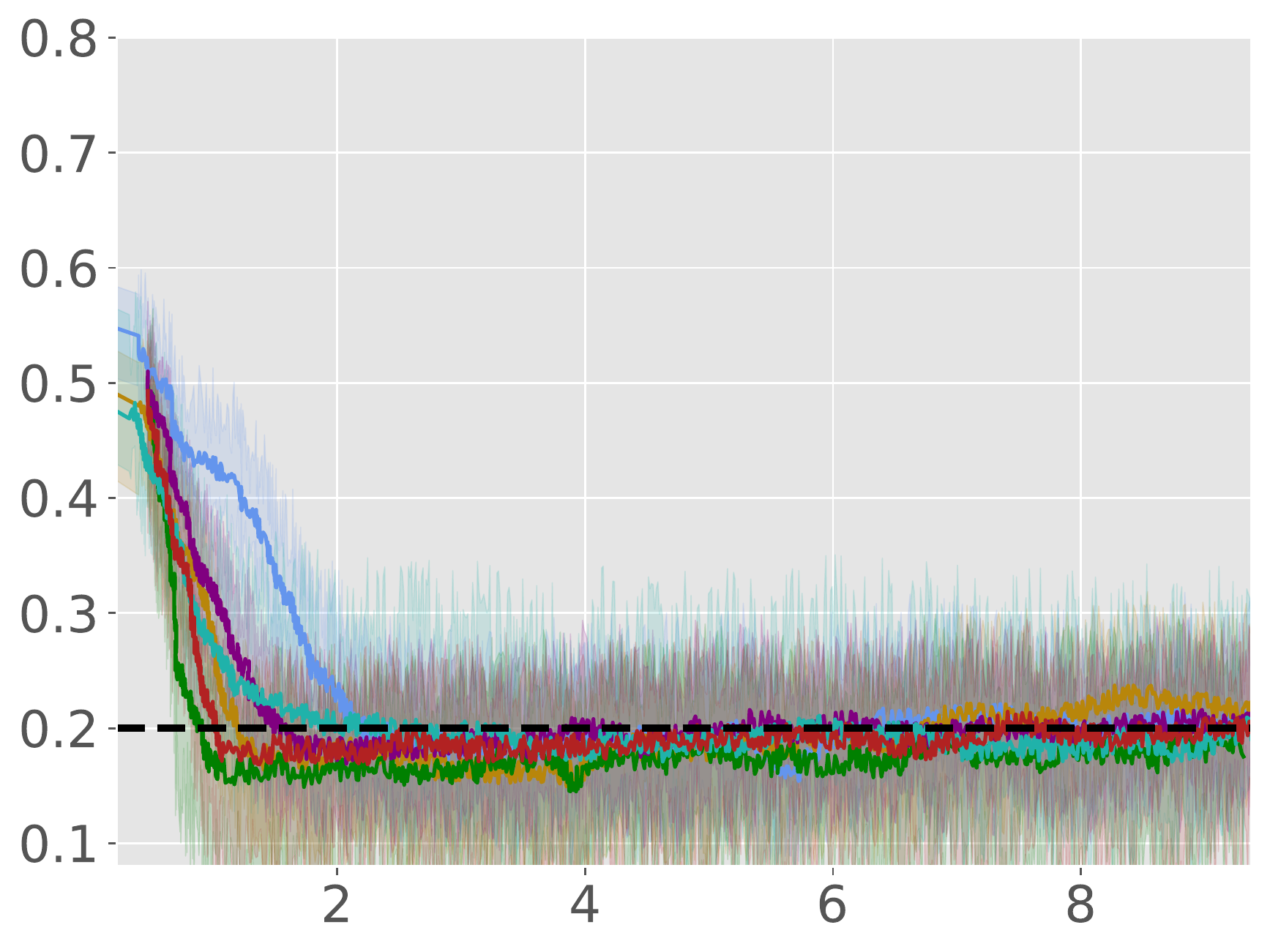}
        \label{fig:point_gather_costs}
    }
    \subfloat[Ant Circle]{
        \includegraphics[width=0.3\textwidth]{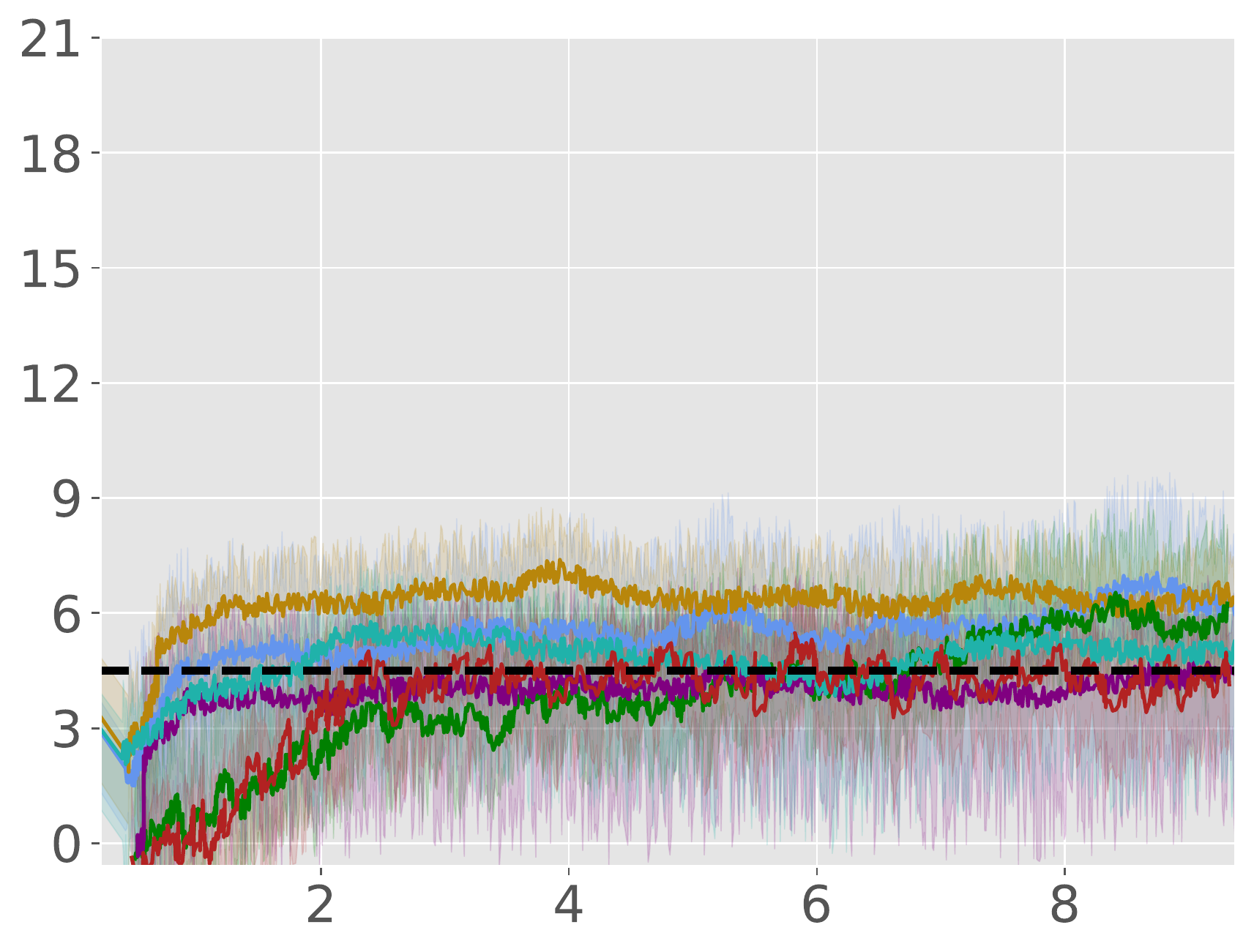}
        \label{fig:ant_circle_costs}
    }

    \caption{The average reward and constraint cost function values vs iterations (in $10^{4}$) learning curves for some algorithm-task pairs. Solid lines in each figure are the empirical means, while the shaded area represents 1 standard deviation, all over 5 runs. The dashed line in constraint plots is the constraint threshold $l$. ATRPO and PPO are tested with constraints, which are included in their Lagrangian formulation.}
    \label{fig:appendix_rewards_costs_comparison}
\end{figure*}

\subsubsection{Recovery Regime Revisited}

In Subsection \ref{subsec:recovery_regime}, we studied the effect of the hyperparameter $t$ for only one task. Figure \ref{fig:hyperparam_t_appendix} shows the performance of ACPO with different values of $t$ in various environments.

\begin{figure*}[ht]
    \hspace{0.25cm} Rewards: \newline
    {
        \includegraphics[width=0.18\textwidth]{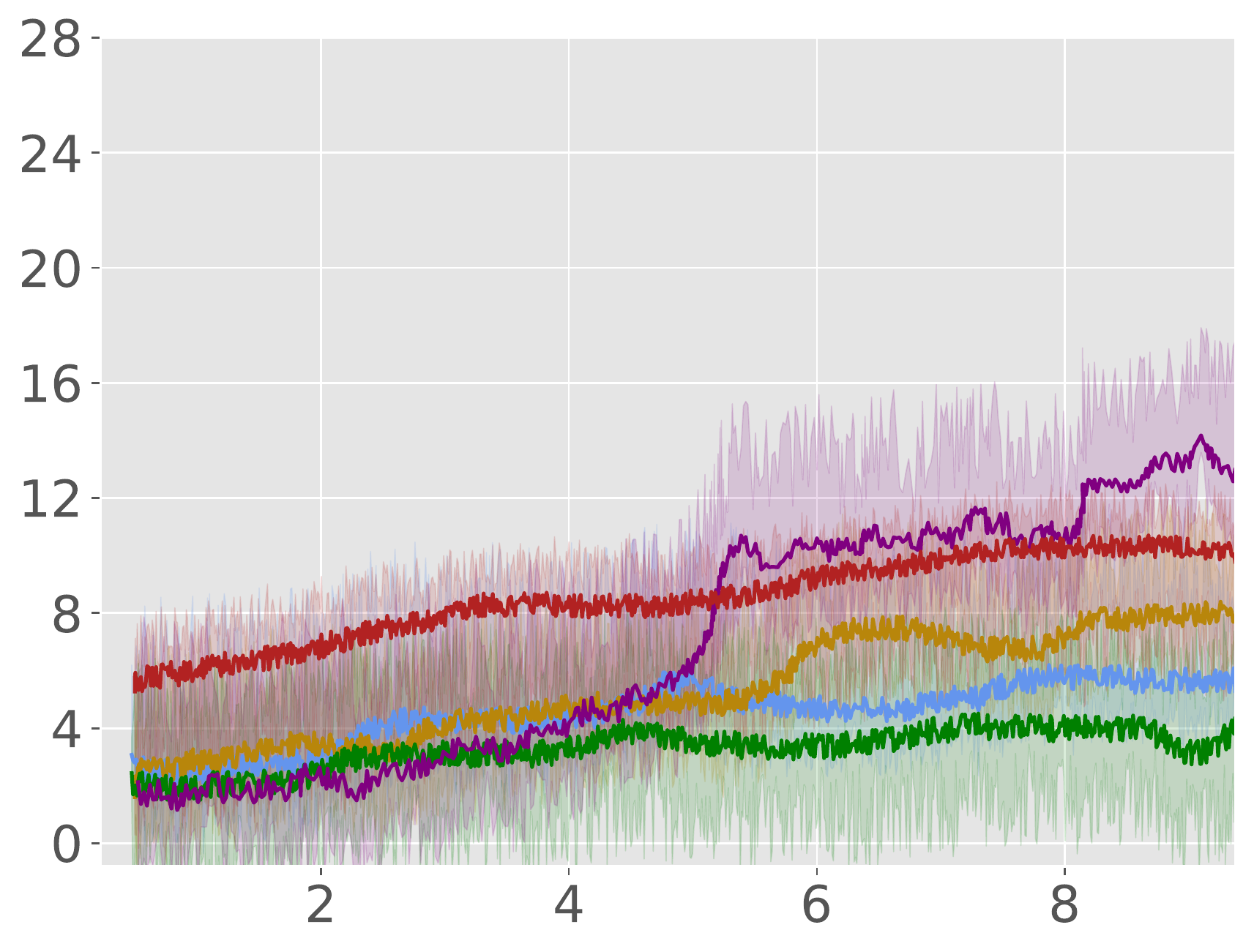}
        \label{fig:point_gather_hyper_rewards}
    }
    {
        \includegraphics[width=0.18\textwidth]{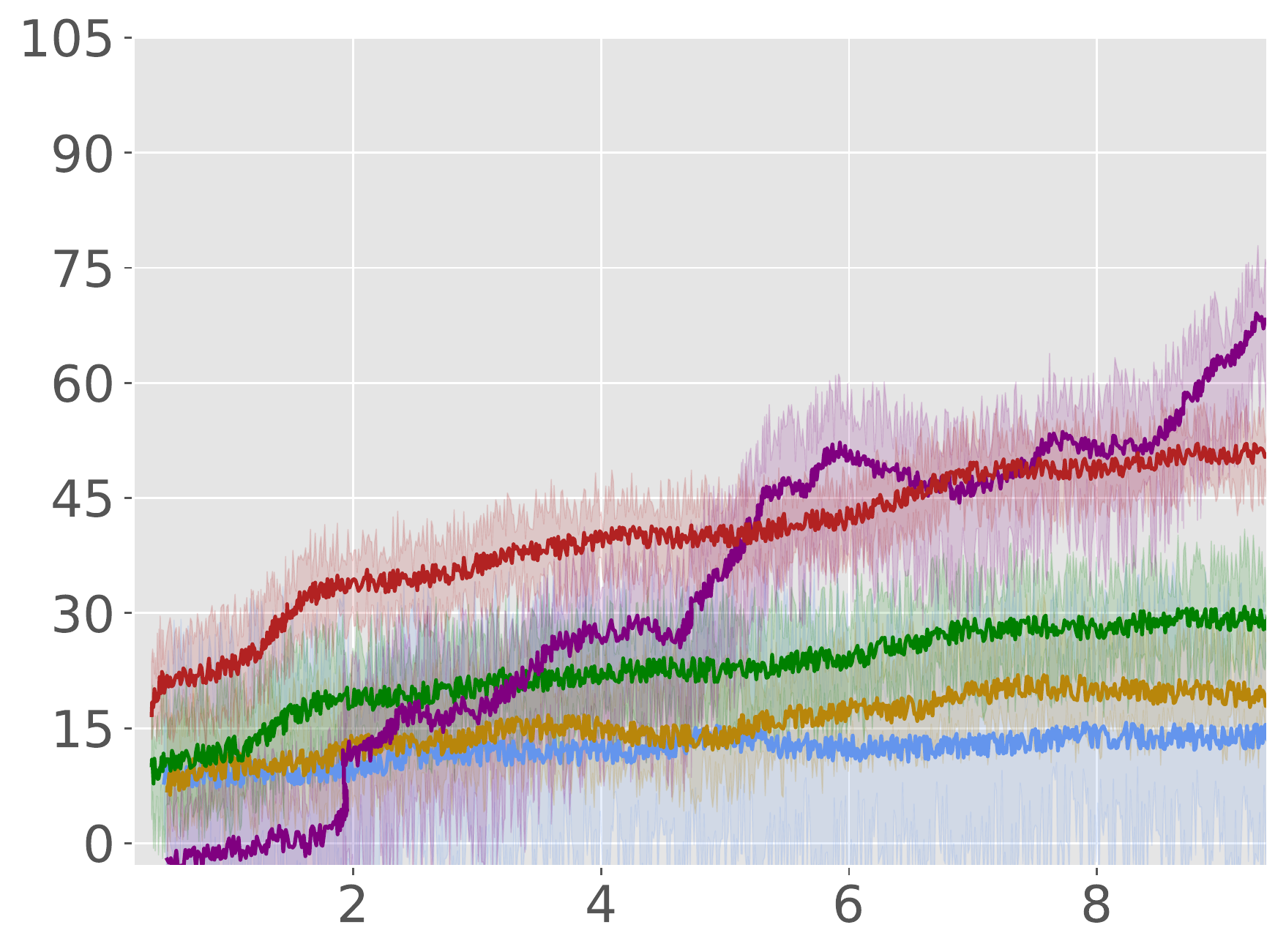}
        \label{fig:ant_circle_hyper_rewards}
    }
    {
        \includegraphics[width=0.185\textwidth]{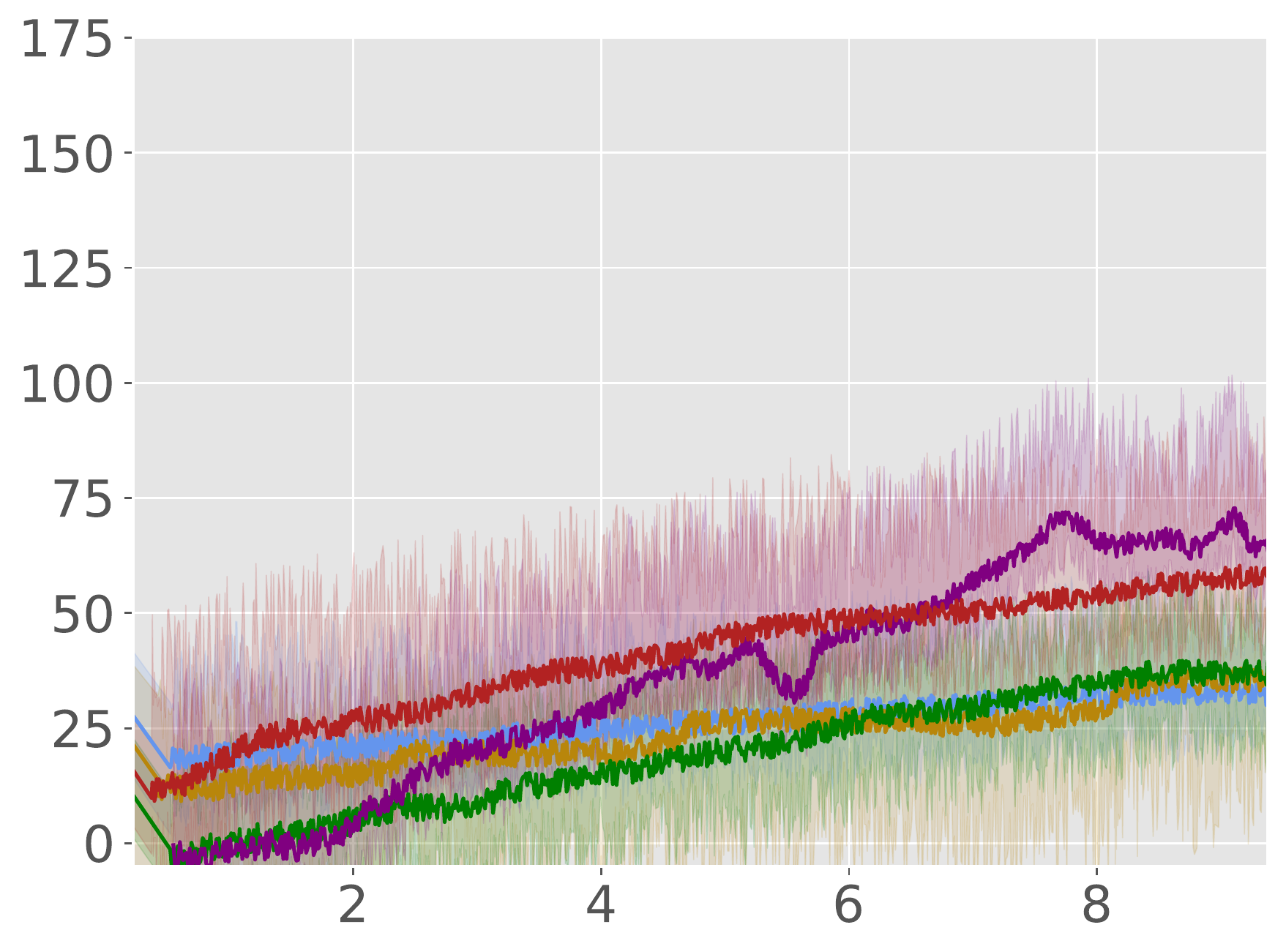}
        \label{fig:ant_gather_hyper_rewards}
    }
    {
        \includegraphics[width=0.185\textwidth]{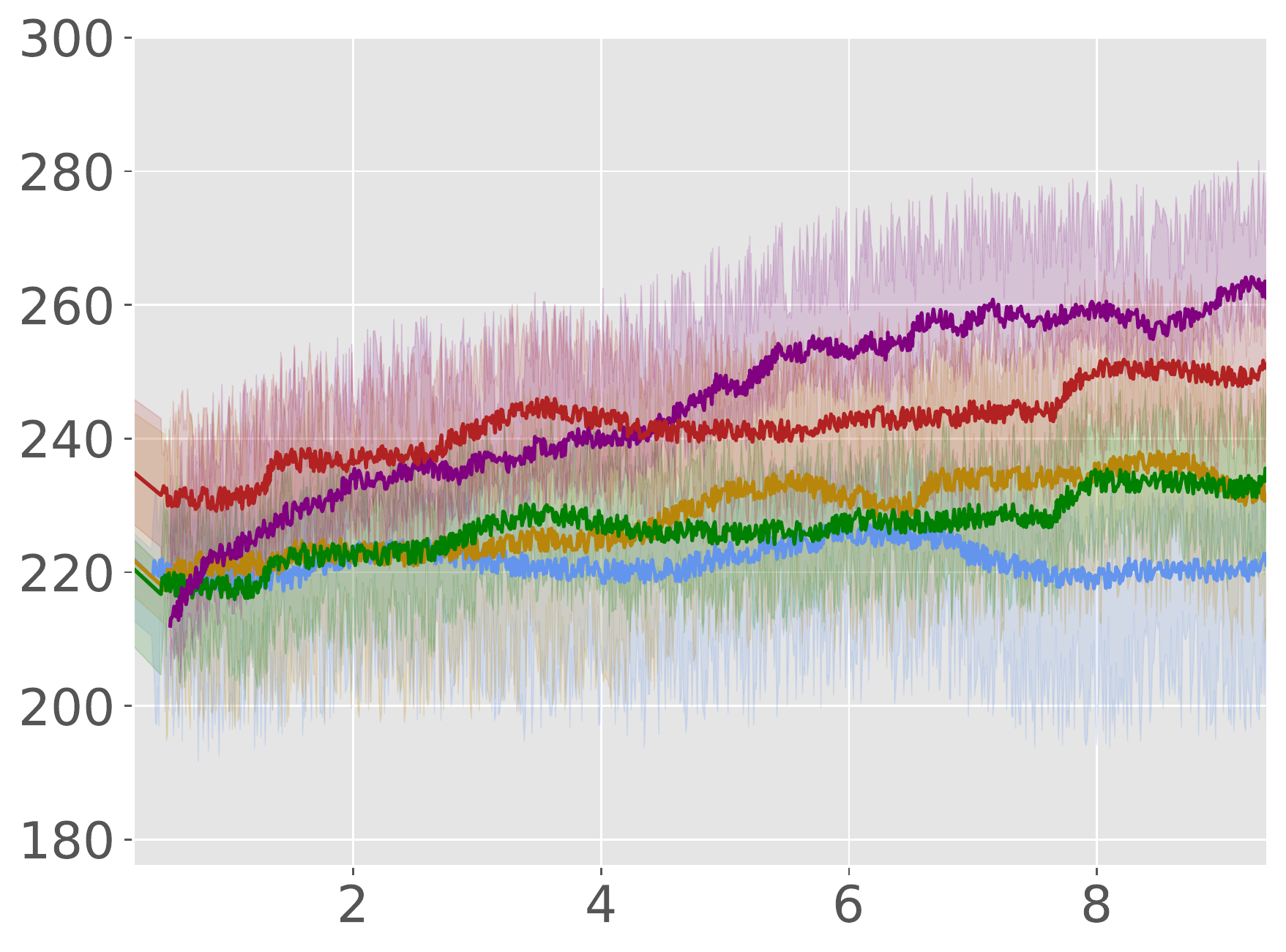}
        \label{fig:bottleneck_hyper_rewards}
    }
    {
        \includegraphics[width=0.185\textwidth]{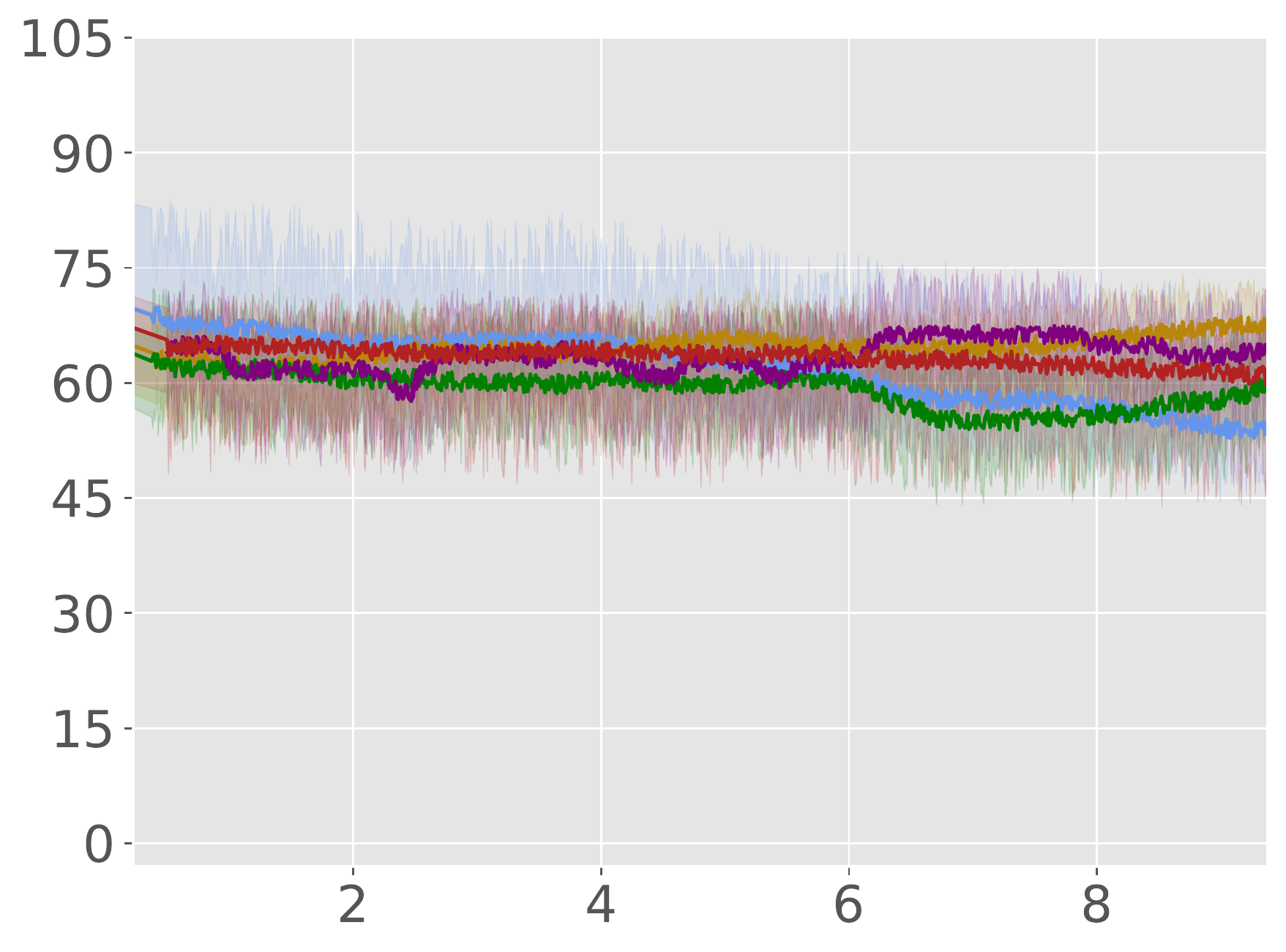}
        \label{fig:grid_hyper_rewards}
    }

    \rule{\linewidth}{0.5pt}
    \hspace{2cm} Constraint values: \hspace{4cm} \includegraphics[width=0.5\textwidth]{images/legend_t_hyperparam.JPG} \newline

    \subfloat[Point Gather]{
        \includegraphics[width=0.185\textwidth]{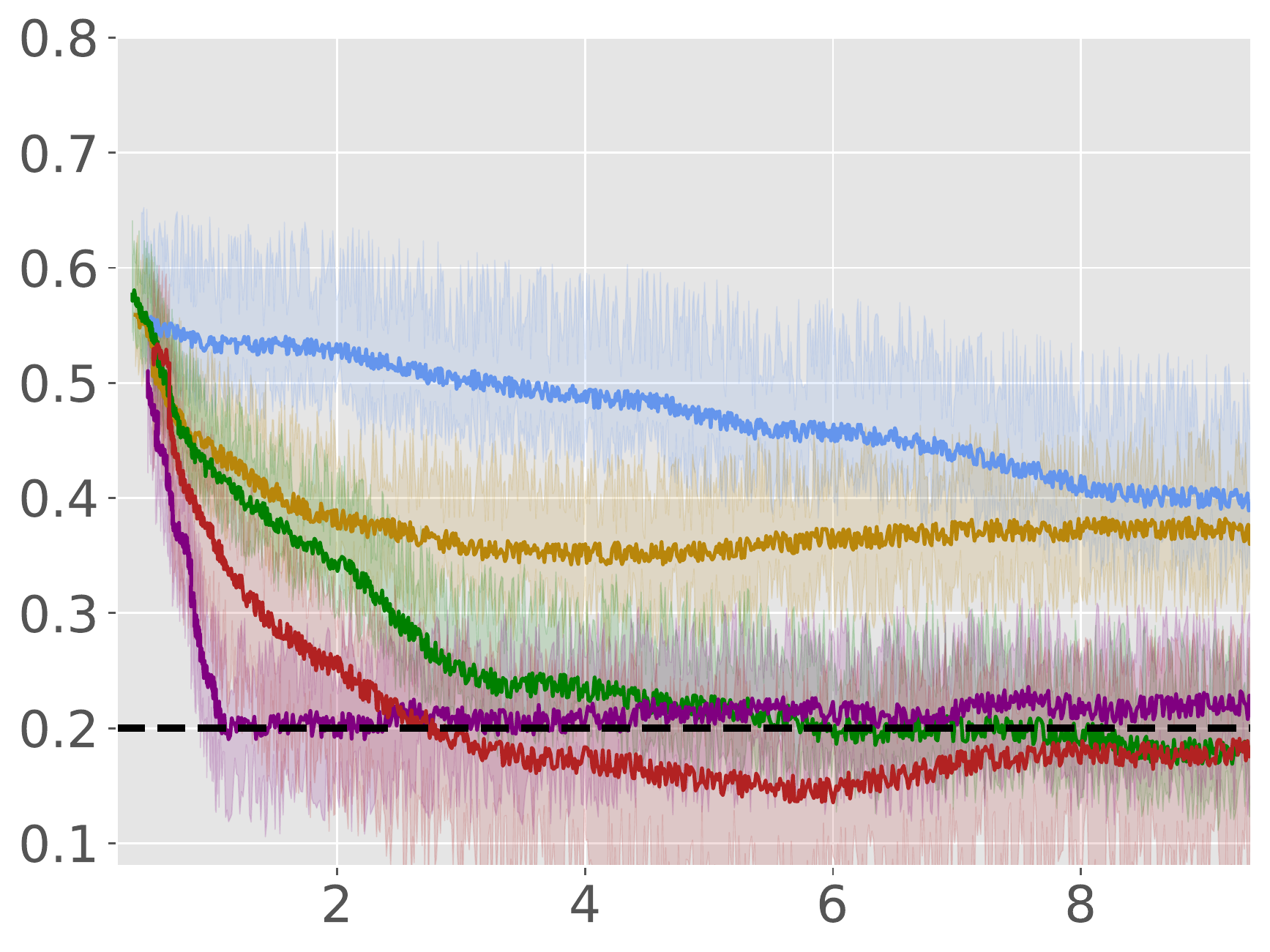}
        \label{fig:point_gather_hyper_costs}
    }
    \subfloat[Ant Circle]{
        \includegraphics[width=0.185\textwidth]{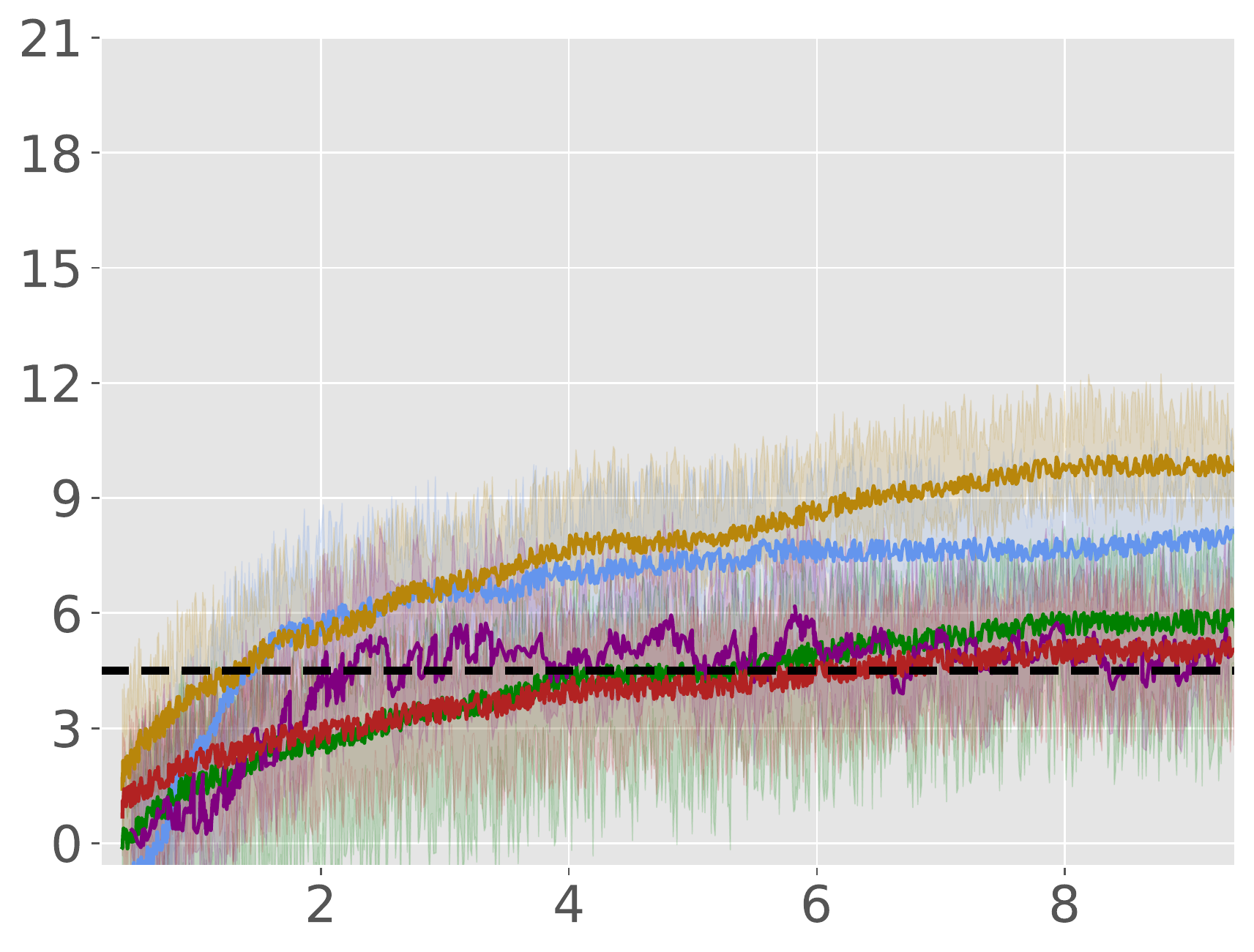}
        \label{fig:ant_circle_hyper_costs}
    }
    \subfloat[Ant Gather]{
        \includegraphics[width=0.185\textwidth]{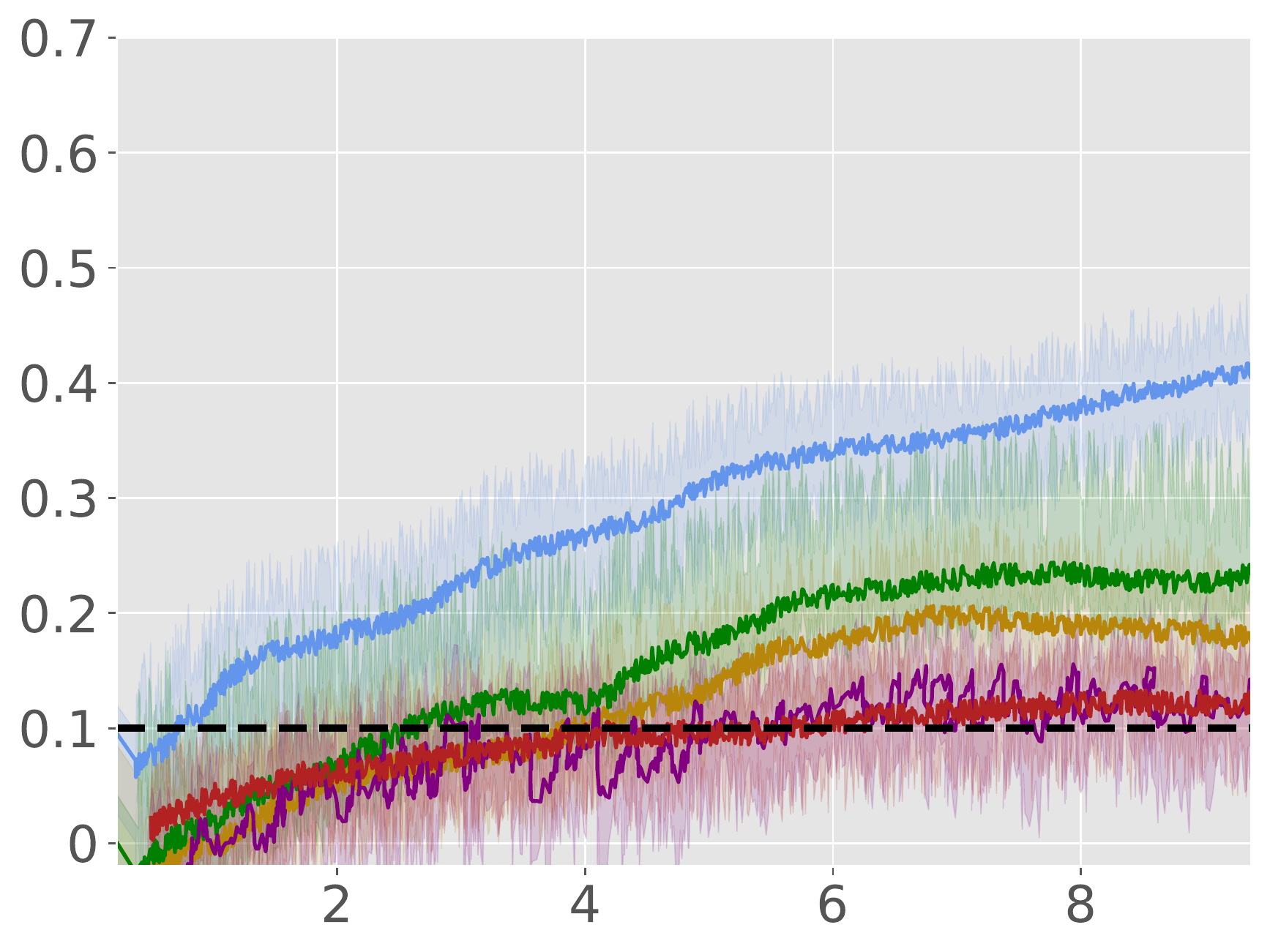}
        \label{fig:ant_gather_hyper_costs}
    }
    \subfloat[Bottleneck]{
        \includegraphics[width=0.185\textwidth]{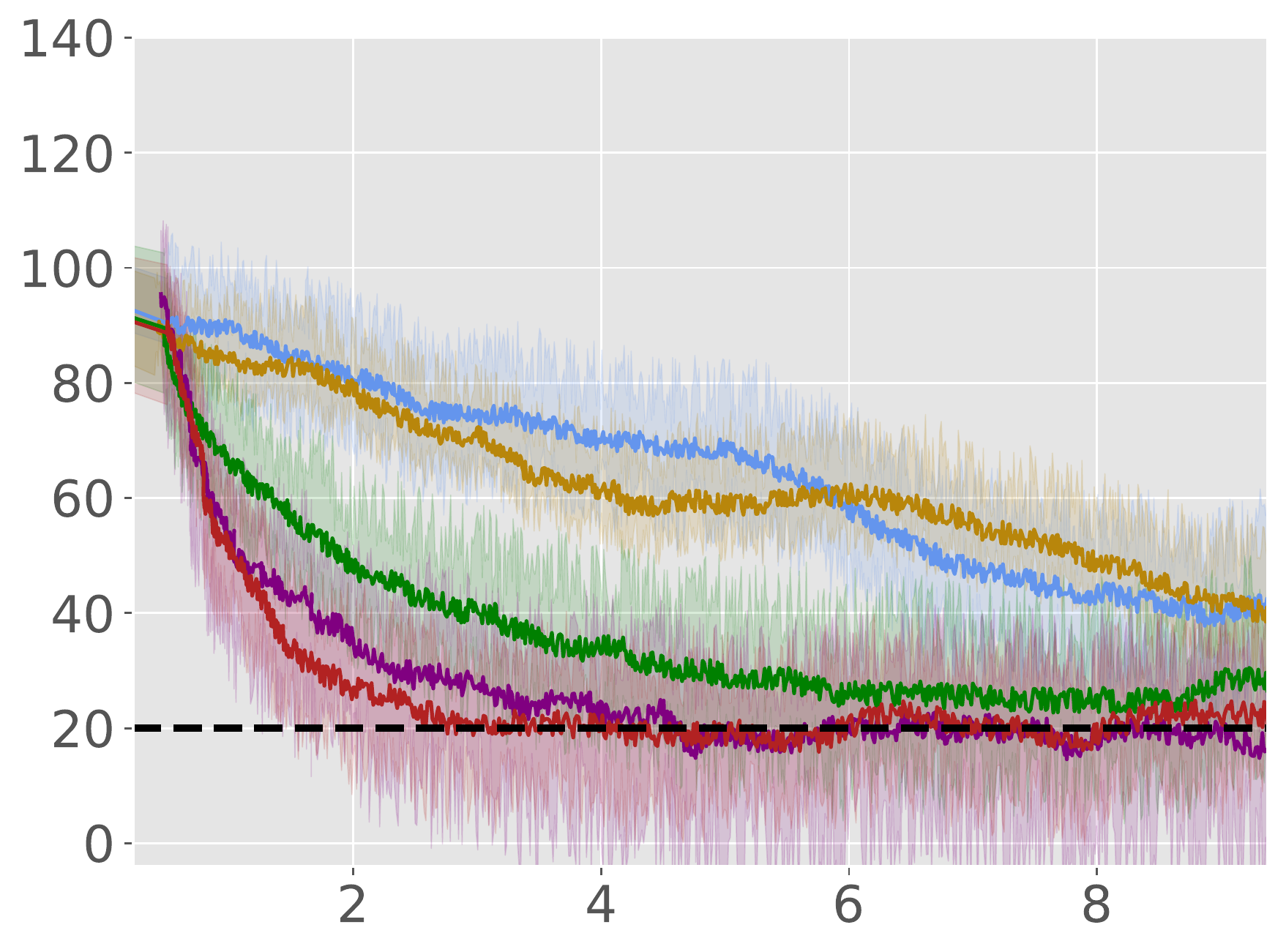}
        \label{fig:bottleneck_hyper_costs}
    }
    \subfloat[Grid]{
        \includegraphics[width=0.185\textwidth]{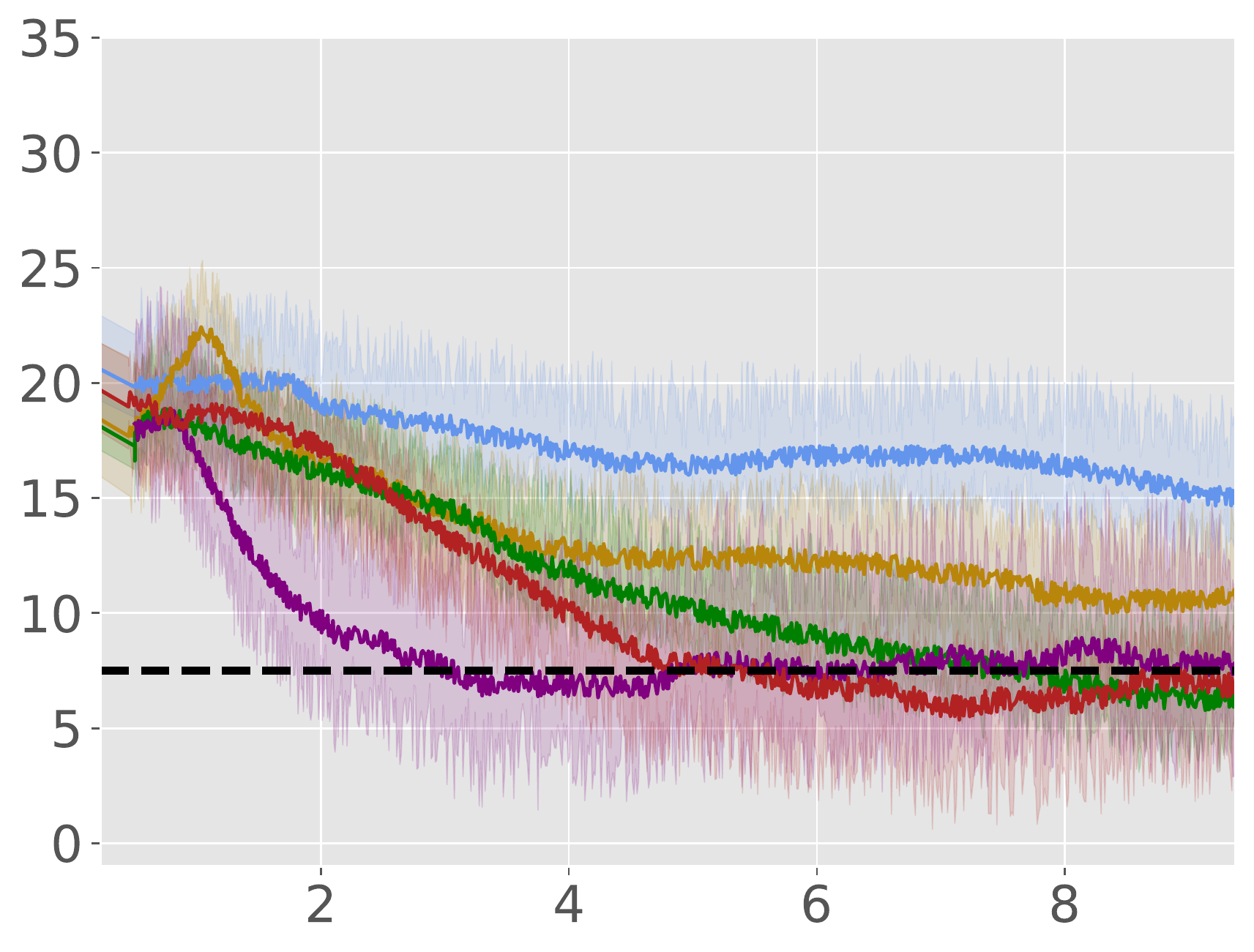}
        \label{fig:grid_hyper_costs}
    }
        \caption[]{Comparison of performance of ACPO with different values of the hyperparameter $t$ in various environment. X-axis is iterations in $10^{4}$.}
    \label{fig:hyperparam_t_appendix}
\end{figure*}

%%%%%%%%%%%%%%%%%%%%%%%%%%%%%%%%%%%%%%%%%%%%%%%%%%%%%%%%%%%%%%%%%%%%%%%%%%%%%%%
%%%%%%%%%%%%%%%%%%%%%%%%%%%%%%%%%%%%%%%%%%%%%%%%%%%%%%%%%%%%%%%%%%%%%%%%%%%%%%%

\end{document}